\documentclass[twoside,10pt]{article}
\usepackage{graphicx}
\usepackage[numbers]{natbib}
\usepackage[T1]{fontenc}
\usepackage{lmodern}

\usepackage{xparse}
\usepackage{amscd,amsmath,amsthm,enumerate,float,amssymb,epsfig,multirow,color,latexsym,mathrsfs,tocloft,listofsymbols,tabu}

\usepackage{algcompatible}
\usepackage{bbm}
\usepackage{booktabs}
\usepackage{subcaption}
\usepackage{graphicx}
\usepackage[titletoc]{appendix}
\usepackage[letterpaper,left=3.81cm,top=2.54cm,bottom=2.54cm,right=2.54cm]{geometry}
\usepackage[nottoc]{tocbibind}

\usepackage{appendix}
\usepackage{chngcntr}
\usepackage{etoolbox}

\usepackage{graphicx}
\usepackage{sidecap}
\usepackage{float}
\usepackage{enumerate,mdwlist}
\usepackage{epsfig}
\usepackage[english]{babel}
\usepackage[bbgreekl]{mathbbol}
\usepackage{mleftright}
\usepackage{colortbl}

\usepackage{subcaption}
\usepackage{mathrsfs}
\usepackage{bm}
\usepackage{latexsym}
\usepackage{mleftright}
\usepackage{color}
\usepackage{xifthen}

\usepackage[titletoc]{appendix}

\usepackage{footnote}

\newlength{\oldparindent}
\usepackage{fancyhdr}
\usepackage{hyperref} 
\hypersetup{
	colorlinks = true,
	linkcolor = blue,
	anchorcolor = blue,
	citecolor = blue,
	filecolor = blue,
	urlcolor = blue
}

\usepackage[ruled,vlined,resetcount,algosection]{algorithm2e}
\usepackage{makeidx}
\usepackage{amsmath}
\usepackage{mathrsfs}
\usepackage{graphicx}

\newcommand{\rr}{{\mathbb{R}}}
\newcommand{\rrflex}[1]{{\ensuremath{\rr^{#1}
}}}
\newcommand{\rrD}{{\rrflex{D}}}
\newcommand{\rrd}{{\rrflex{d}}}

\newcommand{\rrm}{{\rrflex{m}}}
\newcommand{\rrn}{{\rrflex{n}}}

\newcommand{\pp}{{\mathbb{P}}}
\newcommand{\qq}{{\mathbb{Q}}}

\newcommand{\nn}{{\mathbb{N}}}

\newcommand{\fff}{{\mathscr{F}}}	
\newcommand{\ffff}{{\mathfrak{F}}}

\newcommand{\xxx}{\mathcal{X}}

\NewDocumentCommand\deep{}{{\operatorname{deep}}}
\NewDocumentCommand\co{mo}{\operatorname{co}\IfValueT{#2}{_{#2}}\left({#1}\right)}

\NewDocumentCommand\argmin{o}{{\operatorname{argmin}\IfValueT{#1}{_{{#1}}}}}
\NewDocumentCommand\AF{o}{\operatorname{AF}\IfValueT{#1}{
		{
			\left({#1}\right)
		}
}}

\NewDocumentCommand{\ee}{moo}{
	{
		\mathbb{E}_{\IfValueF{#3}{\pp}\IfValueT{#3}{{#3}}}\left[
		{#1}
		\IfValueT{#2}{\mid {#2}}
		\right]
	}
}
\NewDocumentCommand\arch{o}{
{\left(\fff\IfValueT{#1}{_{#1}},\circlearrowleft\IfValueT{#1}{_{#1}}\right)}
}
\NewDocumentCommand\archdeep{oo}{
	{\left(\fff\IfValueT{#1}{_{#1}},\circlearrowleft\IfValueT{#1}{_{#1}},(\Phi_i)_{i \in I}\right)}
}
\NewDocumentCommand\NN{oo}{
	{\mathcal{NN}%
		\IfValueT{#1}{^{#1}}\IfValueF{#1}{^{\arch}}
		\IfValueT{#2}{_{#2}}
	}
}

\NewDocumentCommand\NNshal{oo}{
	{nn%
		_{\IfValueT{#1}{#1}}	
		\IfValueF{#2}{^{\ffff}}\IfValueT{#2}{^{#2}}
	}
}
\NewDocumentCommand\NNho{oo}{
	{\mathcal{HNN}%
		_{R\IfValueT{#1}{#1}}	
		\IfValueF{#2}{^{\ffff}}\IfValueT{#2}{^{#2}}
	}
}

\NewDocumentCommand\NNaff{oo}{
	{\mathcal{NN}%
		_{R,\IfValueT{#1}{#1}}	
		\IfValueF{#2}{^{a}}
	}
}

\newtheorem{thrm}{Theorem}
\newtheorem{defn}{Definition}

\newtheorem{prop}{Proposition}
\newtheorem{cor}{Corollary}
\newtheorem{lem}{Lemma}
\newtheorem{ex}{Example}
\newtheorem{rremark}{Remark}

\NewDocumentCommand{\prodd}{oo}{
	\overset{{#2}}{
		\underset{{#1}}{
			\circlearrowleft
		}
	}
}

\usepackage{titling}
\newcommand{\subtitle}[1]{%
	\posttitle{%
		\par\end{center}
	\begin{center}\large#1\end{center}
	\vskip0.5em}%
}
\begin{document}
\title{The Universal Approximation Property}
\subtitle{Characterization, Construction, Representation, and Existence}

	\author{Anastasis Kratsios \thanks{
		Department of Mathematics, (ETH) Eidgen\"{o}ssische Technische Hochschule Z\"{u}rich, HG G 32.3.
		Tel.: +41 44 632 3751,
		anastasis.kratsios@math.ethz.ch		
		ORCID: 0000-0001-6791-3371}
}

\date{%
	November 27$^{th}$ 2020 %
}

\maketitle

\begin{abstract}
The universal approximation property of various machine learning models is currently only understood on a case-by-case basis, limiting the rapid development of new theoretically justified neural network architectures and blurring our understanding of our current models' potential.  This paper works towards overcoming these challenges by presenting a characterization, a representation, a construction method, and an existence result, each of which applies to any universal approximator on most function spaces of practical interest. Our characterization result is used to describe which activation functions allow the feed-forward architecture to maintain its universal approximation capabilities when multiple constraints are imposed on its final layers and its remaining layers are only sparsely connected. These include a rescaled and shifted Leaky ReLU activation function but not the ReLU activation function. Our construction and representation result is used to exhibit a simple modification of the feed-forward architecture, which can approximate any continuous function with non-pathological growth, uniformly on the entire Euclidean input space.  This improves the known capabilities of the feed-forward architecture. 
\end{abstract}



\section{Introduction}\label{s_Intro}
Neural networks have their organic origins in \cite{mcculloch1943logical} and in~\cite{rosenblatt1958perceptron}, wherein the authors pioneered a method for emulating the behavior of the human brain using digital computing.  Their mathematical roots are traced back to Hilbert's $13^{th}$ problem, which postulated that all high-dimensional continuous functions are a combination of univariate continuous functions.  

Arguably the second major wave of innovation in the theory of neural networks happened following the \textit{universal approximation theorems} of \cite{hornik1990universal}, \cite{Cybenko}, and of \cite{hornik1991approximation}, which merged these two seemingly unrelated problems by demonstrating that the feed-forward architecture is capable of approximating any continuous function between any two Euclidean spaces, uniformly on compacts.  This series of papers initiated the theoretical justification of the empirically observed performance of neural networks, which had up until that point only been justified by analogy with the Kolmogorov-Arnold Representation Theorem of \citep{kolmogoroff1957representation}.  

Since then the universal approximation capabilities, of a limited number of neural network architectures, such as the feed-forward, residual, and convolutional neural networks has been solidified as a cornerstone of their approximation success.  This, coupled with the numerous hardware advances have led neural networks to find ubiquitous use in a number of areas, ranging from biology, see \cite{webb2018deep,eraslan2019deep}, to computer vision and imaging, see \cite{plis2014deep,akhtar2018threat}, and to mathematical finance, see \cite{buehler2019deep,becker2019deep,cuchiero2020generative,kratsios2017arbitrage,horvath2020deep}.  As a result, a variety of neural network architectures have emerged with the common thread between them being that they describe an algorithmically generated set of complicated functions built by combining elementary functions in some manner. 

However, the case-by-case basis for which the universal approximation property is currently understood limits the rapid development of new theoretically justified architectures.  This paper works at overcoming this challenges by directly studying universal approximation property itself in the form of far-reaching characterizations, representations, construction methods, and existence results applicable to most situations encounterable in practice.

The paper's contributions are organized as follows.  Section~\ref{s_prelim} overviews the analytic, topological, and learning-theoretic background required in formulating the paper's results.  

Section~\ref{s_1_charancterization} contains the paper's main results.  These include a characterization, a representation result, a construction theorem, and existence result applicable to any universal approximator in most function spaces of practical interest.  The characterization result shows that an architecture has the UAP on a function space if and only if that architecture implicitly decomposes the function space into a collection of separable Banach subspaces, whereon the architecture contains the orbit of a topologically transitive dynamical system.   Next, the representation result shows that any universal approximator can always be approximately realized as a transformation of the feed-forward architecture.  This result reduces the problem of constructing new universal architectures for identifying the correct transformation of the feed-forward architecture for the given learning task.  The construction result gives conditions on a set of transformations of the feed-forward architecture, guaranteeing that the resultant is a universal approximator on the target function space.  Lastly, we obtain a general existence and representation result for universal approximators generated by a small number of functions applicable to many function spaces.

Section~\ref{s_Implications} then focuses the main theoretical results to the feed-forward architecture.  Our characterization result is used to exhibit the dynamical system representation on the space of continuous functions by composing any function with an additional deep feed-forward layer whose activation function is continuous, injective, and has no fixed points.  Using this representation, we show that the set of all such deep feed-forward networks constructed through this dynamical system maintain its universal approximation property even when constraints are imposed on the network's final layers or when sparsity is imposed on the network's connections' initial layers.  In particular, we show that feed-forward networks with ReLU activation function fail these requirements, but a simple affine transformation of the Leaky-ReLU activation function is of this type.  We provide a simple and explicit method for modifying most commonly used activation functions into this form.  We also show that the conditions on the activation function are sharp for this dynamical system representation to have the desired topological transitivity properties.  

As an application of our construction and representation results, we build a modification of the feed-forward architecture which can uniformly approximate a large class of continuous functions which need not vanish at infinity.  This architecture approximates uniformly on the entire input space and not only on compact subsets thereof.  This refines the known guarantees for feed-forward networks (see \cite{leshno1993multilayer,kidger2019universal}) which only guarantee uniform approximation on compacts subsets of the input space, and consequentially, for functions vanishing at infinity.  As a final application of the results, the existence theorem is then used to provide a representation of a small universal approximator on $L^{\infty}(\rr)$, which provides the first concrete step towards obtaining a tractable universal approximator thereon.  
\section{Background and Preliminaries}\label{s_prelim}
This section overviews the analytic, topological, and learning-theoretic background used to in this paper.  
\subsection*{Metric Spaces}\label{ss_prelim_pt_met_spaces}
Typically, two points $x,y\in \rrm$ are thought of as being near to one another if $y$ belongs to the \textit{open ball} with radius $\delta> 0$ centered about $x$ defined by $\operatorname{Ball}_{\rrm}(x,\delta)\triangleq \{z \in \rrm:\,\|x-z\|<\delta\}$, where $(x,z)\mapsto\|x-z\|$ denotes the Euclidean distance function.  The analogue can be said if we replace $\rrm$ by a set $X$ on which there is a \textit{distance function} $d_X:X \times X\rightarrow [0,\infty)$ quantifying the closeness of any two members of $X$.  Many familiar properties of the Euclidean distance function are axiomatically required of $d_X$ in order to maintain many of the useful analytic properties of $\rrm$; namely, $d_X$ is required to satisfy the triangle inequality, symmetry in its arguments, and it vanishes precisely when its arguments are identical.  As before, two points $x,y \in X$ are thought of as being close if they belong to the same \textit{open ball}, $\operatorname{Ball}_X(x,\delta)\triangleq \{z\in X:\,d_X(x,z)<\delta\}$ where $\delta> 0$.  
Together, the pair $(X,d_X)$ is called a \textit{metric space}, and this simple structure can be used to describe many familiar constructions prevalent throughout learning theory.  We follow the convention of only denoting $(X,d_X)$ by $X$ whenever the context is clear.  

\begin{ex}[Spaces of Continuous Functions]\label{ex_cnt_functions}
For instance, the universal approximation theorems of \cite{hornik1989multilayer,leshno1993multilayer,kidger2019universal,park2020minimum} describe conditions under which any continuous function from $\rrm$ to $\rrn$ can be approximated by a feed-forward neural network.  The distance function used to formulate their approximation results is defined on any two continuous functions $f,g:\rrm\rightarrow \rrn$ via
$$
d_{ucc}(f,g)\triangleq 
\sum_{k=1}^{\infty}
\frac{
	\sup_{x \in [-k,k]^m} \|f(x)-g(x)\|
}{
	2^k\left(1+
	\sup_{x \in [-k,k]^m} \|f(x)-g(x)\|
	\right)
}
.
$$
In this way, the set of continuous functions from $\rrm$ to $\rrn$ by $C(\rrm,\rrn)$ is made into a metric space when paired with $d_{ucc}$.  
In what follows, we make the convention of denoting $C(X,\rr)$ by $C(X)$.  
\end{ex}
\begin{ex}[Space of Integrable Functions]\label{ex_space_integrable}
Not all functions encountered in practice are continuous, and the approximation of discontinuous functions by deep feed-forward networks is studied in \cite{Hannin,ZhouWidthUniversalApproximationReLU} for functions belonging to the space $L^p_{\mu}(\rrm,\rrn)$.  
Briefly, elements of $L^p_{\mu}(\rrm,\rrn)$ are equivalence classes of Borel measurable $f:\rrm\rightarrow \rrn$, identified up to $\mu$-null sets, for which the norm
$$
\|f\|_{p,\mu}\triangleq \left(\int_{x \in \rrm} \|f(x)\|^p d\mu(x)\right)^{\frac1{p}}
$$
is finite; here $\mu$ is a fixed Borel measure on $\rrm$ and $1\leq p<\infty$.  We follow the convention of denoting $L^p_{\mu}(\rrm,\rr)$ by $L^p(\rrm)$ when $\mu$ is the Lebesgue measure on $\rrm$.  
\end{ex}
Unlike $C(\rrm,\rrn)$, the distance function on $L^p_{\mu}(\rrm,\rrn)$ is induced through a norm via $(f,g)\mapsto \|f-g\|_{p,\mu}$.  Spaces of this type simultaneously carry compatible metric and vector spaces structures.  Moreover, in such a space, if every sequence converges whenever its pairwise distances asymptotically tend to zero, then the space is called a \textit{Banach space.}  The prototypical Banach space is $\rrm$.  

Unlike Banach spaces or the space of Example~\ref{ex_cnt_functions}, general metric spaces are non-linear.  That is, there is no meaningful notion of addition, scaling, and there is no singular reference point analogous to the $0$ vector.  Examples of non-linear metric spaces arising in machine learning are shape spaces used in neuroimaging applications (see \cite{fletcher2003statistics}), graphs and trees arising in structured and hierarchical learning (see \cite{MR4070735,ganea2018hyperbolic}), and spaces of probability measures appearing in adversarial approaches to learning (see \cite{zhang2019self}).  

The lack of a reference point may always be overcome by artificially declaring a fixed element of $X$, denoted by $0_X$, to be the central point of reference in $X$.  In this case, the triple $(X,d_X,0_X)$, is called a \textit{pointed metric space}.  We follow the convention of denoting the triple by $X$, whenever the context is clear.  For pointed metric spaces $X$ and $Y$, the class of functions $f:X\rightarrow Y$ satisfying $f(0_X)=0_Y$ and $\|f(x_1)-f(x_2)\|\leq L\|x_1-x_2\|$, for some $L>0$ and every $x_1,x_2 \in X$, is denoted by $\operatorname{Lip}_0(X,Y)$ and this class is understood as mapping the structure of $X$ into $Y$ without too large of a distortion.  In the extreme case where an $f\in \operatorname{Lip}_0(X,Y)$ perfectly respects the structure of $X$, i.e$.:$ when $\|f(x_1)-f(x_2)\|=\|x_1-x_2\|$, we call $f$ a \textit{pointed isometry}.  In this case, $f(X)$ represents an exact copy of $X$ within $Y$.

The remaining non-linear aspects of a general metric space pose no significant challenge and this is due to the following linearization feature map of \cite{ArensEells}.  Since its inception, the following method has found notable applications in clustering \cite{ArensEelssJMLR} and in optimal transport \cite{ambrosio2019linear}.  In particular, the later connects this linearization procedure with optimal transport approaches to adversarial learning of \cite{pmlrv70arjovsky17a,xu2020cot}.  

\begin{ex}[Free-Space over $X$]\label{ex_Free_Space}
We follow the formulation described in \cite{ambrosio2019linear}.  
Let $X$ be a metric space and for any $x \in X$, let $\delta_x$ be the (Borel) probability measure assigning value $1$ to any $\operatorname{Ball}_X({y,\epsilon})\subseteq X$ if $x \in \operatorname{Ball}_X(y,\epsilon)$ and $0$ otherwise.  
The Free-space over $X$ is the Banach space $B(X)$ obtained by completing the vector space $\left\{
\sum_{n=1}^N \alpha_n \delta_{x_n}:\, a_n \in \rr,\,x_n \in X, n=1,\dots,N,\, N\in \nn_+
\right\}$ with respect to the following%
\begin{equation}
\left\|
\sum_{i=1}^n \alpha_i x_i 
\right\|_{B(X)} \triangleq \underset{\|f\|\leq 1;\, f \in Lip_0(X,\rr)}{\sup} 
\sum_{i=1}^n \alpha_i f(x_i)
\label{eq_Wasserstein_norm}
.
\end{equation}
As shown in \citep[Proposition 2.1]{GeodefroyKaltonRemembering}, the map $\delta^X: x \mapsto \delta_{x}$ is a (non-linear) isometry from $X$ to $B(X)$.  As shown in \cite{WeaverNice}, the pair $(B(X),\delta^X)$ is characterized by the following linearization property: whenever $f \in \operatorname{Lip}_0(X,Y)$ and $Y$ is a Banach space then there exists a unique continuous linear map satisfying
\begin{equation}
f = F\circ \delta^X
\label{eq_linearization_property_unviersal_property}
.
\end{equation}
Thus, $\delta^X:X\rightarrow B(X)$ can be interpreted as a minimal isometric linearizing \textit{feature map}.  
\end{ex}
Sometimes the feature map $\delta^X$ can be continuously inverted from the left.  In \cite{GeodefroyKaltonRemembering} any continuous map $\rho:B(X)\rightarrow X$ is called a \textit{barycenter} if it satisfies $\rho\circ \delta^X=1_{X}$, where $1_{X}$ is the identity on $X$.  
Following \cite{godefroy2003lipschitz}, if a barycenter exists then $X$ is called \textit{barcycentric}.  Examples of barycentric spaces are Banach spaces \cite{GodefroyLipfReeBan}, Cartan-Hadamard manifolds described (see \citep[Corollary 6.9.1]{jost2008riemannian}), and other structures described in \cite{basso2020extending}.  Accordingly, many function spaces of potential interest contain a dense barycentric subspace.  When the context is clear, we follow the convention of denoting $\delta^X$ simply by $\delta$.  

\subsection*{Topological Background}\label{ss_top_background}
Rather than using open balls to quantify closeness, it is often more convenient to work with \textit{open subsets} of $X$; where $U\subseteq X$ is said to be open whenever every point $x \in U$ belongs to some open ball $B_X({x,\delta})$ contained in $U$.  This is because open sets have many desirable properties; for example, a convergent sequence contained in the complement of an open set must also have its limit in that open set's complement.  Thus, the complement of open sets are often called \textit{closed sets} since their limits cannot escape them.  

Unfortunately, many familiar situations arising in approximation theory cannot be described by a distance function.  For example, there is no distance function describing the point-wise convergence of a sequence of functions $\{f_n\}_{n \in \nn}$ on $\rrm$ to any other such function $f$ (for details \citep[page 362]{NagataTopologyGen}).  In these cases, it is more convenient to work directly with \textit{topologies}.  A topology $\tau$ is a collection of subsets of a given set $X$ whose members are declared as being \textit{open} if $\tau$ satisfies certain algebraic conditions emulating the basic properties of the typical open subsets of $\rrm$ (see \citep[Chapter 2]{munkres2014topology}).  Explicitly, we require that $\tau$ contain the empty set $\emptyset$ as well as the entire space $X$, we require that the arbitrary union of subsets of $X$ belonging to $\tau$ also belongs to $\tau$, and we require that finite intersections of subsets of $X$ belonging to $\tau$ also be a member of $\tau$.  A \textit{topological space} is a pair of a set $X$ and a topology $\tau$ thereon.  We follow the convention of denoting topological spaces with the same symbol as their underlying set.  

Most universal approximation theorems \cite{Cybenko,leshno1993multilayer,kidger2019universal} guarantee that a particular subset of $C(\rrm,\rrn)$ is \textit{dense} therein.  In general, $A\subseteq X$ is dense if the smallest closed subset of $X$ containing $A$ is $X$ itself.  Topological spaces containing a dense subset which can be put in a 1-1 correspondence with the natural numbers $\nn$ is called a \textit{separable} space.  Many familiar spaces are separable, such as $C(\rrm)$ and $\rrm$.  

A function $f:\rrm\rightarrow \rrn$ is thought of as continuously depending on its inputs if small variations in its inputs can only produce small variations in its outputs; that is, for any $x \in \rrm,$ $\epsilon>0$ there exists some $\delta_{}>0$ such that 
$
f^{-1}\left[\operatorname{Ball}_{\rrn}({f(x),\epsilon})\right] \subseteq \operatorname{Ball}_{\rrm}({x,\delta}).  
$
It can be shown, see \cite{munkres2014topology}, that this condition is equivalent to requiring that the pre-image $f^{-1}[U]$ of any open subset $U$ of $\rrn$ is open in $\rrm$.  This reformulation means that open sets are preserved under the inverse-image of continuous functions, and it lends itself more readily to abstraction.  Thus, a function $f: X\rightarrow Y$ between arbitrary topological spaces $X$ and $Y$ is continuous if $f^{-1}[U]$ is open in $X$ whenever $U$ is open in $Y$.  If $f$ is a continuous bijection and its inverse function $f^{-1}:Y\rightarrow X$ is continuous, then $f$ is called a \textit{homeomorphism} and $X$ and $Y$ are thought of as being topologically identical.  If $f$ is a homeomorphism onto its image, $f$ is an \textit{embedding}.

We illustrate the use of homeomorphisms with a learning theoretic example.  Many learning problems encountered empirically benefit from feature maps modifying the input a of learning model; for example, this is often the case with kernel methods (see \cite{MR2274454,MR2520802,MR2426053}), in reservoir computing (see \cite{MR4048990,cuchiero2020discretetime}), and in geometric deep learning (see \cite{MR3104017,MR4070735}).  Recently, in \cite{kratsios2020non}, it was shown that, a feature map $\phi:X\rightarrow \rrm$ is continuous and injective if and only if the set of all functions $f\circ \phi \in C(X)$, where ${f}\in C(\rrm)$ is a deep feed-forward network with ReLU activation, is dense in $C(X)$.  A key factor in this characterization is that the map $\Phi:C(\rrm)\rightarrow C(X)$, given by $f\mapsto f\circ \phi$, is an embedding if $\phi$ is continuous and injective.  

The above example suggests that our study of an architecture's approximation capabilities is valid on any topological space which can be mapped homeomorphically onto a well-behaved topological space.  For us, a space will be well-behaved if it belongs to the broad class of Fr\'{e}chet spaces.  Briefly, these spaces have compatible topological space and vector space structures, meaning that the basic vector space operations such as addition, inversion, and scalar multiplication are continuous; furthermore, their topology is induced by a complete distance function which is invariant under translation and satisfies an additional technical condition described in \citep[Section 3.7]{OsborneLCSs2014}.  The class of Fr\'{e}chet spaces encompass all Hilbert and Banach spaces and they share many familiar properties with $\rrm$.  Relevant examples of a Fr\'{e}chet space are $C(\rrm,\rrn)$, the free-space $B(X)$ over any pointed metric space, and  $L^1_{\mu}(\rrm,\rrn)$. 

\subsection*{Universal Approximation Background}
In the machine learning literature, universal approximation refers to a model class' ability to generically approximate any member of a large topological space whose elements are functions, or more rigorously, equivalence classes of functions.  Accordingly, in this paper, we focus on a class of topological spaces which we call \textit{function spaces}.  In this paper, a function space $\xxx$ is a topological space whose elements are equivalence classes of functions between two sets $X$ and $Y$.  For example, when $X=\rr=Y$ then $\xxx$ may be $C(\rr)$ or $L^p(\rr)$.  We refer to $\xxx$ as a \textit{function space between X and Y} and we omit the dependence to $X$ and $Y$ if it is clear from the context.  

The elements in $\xxx$ are called \textit{functions}, whereas functions between sets are referred to as set-functions.  By a \textit{partial function} $f:X\to Y$ we mean a binary relation between the sets $X$ and $Y$ which attributes at-most one output in $Y$ to each input in $X$.

\paragraph{Notational Conventions}
	The following notational conventions are maintained throughout this paper.  
	Only non-empty outputs of any partial function $f$ are specified.  
	We denote the set of positive integers by $\nn^+$.  We set $\nn\triangleq \nn^+\cup\{0\}$.  For any $n \in \nn^+$, the $n$-fold Cartesian product of a set $A$ with itself is denoted by $A^n$.   For $n \in \nn$, we denote the $n$-fold composition of a function $\phi:X\rightarrow X$ with itself by $\phi^n$ and the $0$-fold composition $\phi^0$ is defined to be the identity map on $X$.  

\begin{defn}[Architecture]\label{defn_arch}
	Let $\xxx$ be a function space.  An architecture on $\xxx$ is a pair $\arch$ of a set of set-functions $\fff$ between (possibly different) sets and a partial function 
	$
	\circlearrowleft: \bigcup_{J \in \nn} \fff^J \rightarrow \xxx
	,
	$ 
	satisfying the following non-triviality condition: there exists some $f \in \xxx$, $J \in \nn^+$, and $f_1,\dots,f_J \in \fff$ satisfying
	\begin{equation}
	f= \circlearrowleft\left((f_j)_{j=1}^J\right) \in \xxx
	\label{eq_arch_defn}
	.
	\end{equation}
	The set of all functions $f$ in $\xxx$ for which there is some $J\in \nn^+$ and some $f_1,\dots,f_J \in \fff$ satisfying the representation~\eqref{eq_arch_defn} is denoted by $\NN$.  
\end{defn}
Many familiar structures in machine learning, such as convolutional neural networks, trees, radial basis functions, or various other structures can be formulated as architectures.  To fix notation and to illustrate the scope of our results we express some familiar machine learning models in the language of Definition~\ref{defn_arch}.
\begin{ex}[Deep Feed-Forward Networks]\label{ex_fully_connected_DffNNS}
	Fix a continuous function $\sigma:\rr\to \rr$, denote component-wise composition by $\bullet$, and let $\operatorname{Aff}(\rrd,\rrD)$ be the set of affine functions from $\rrd$ to $\rrD$.  
	Let $\xxx=C(\rrm,\rrn)$, $\fff\triangleq \bigcup_{d_1,d_2,d_3 \in \nn}
	\left\{
	\left(W_2, W_1\right):\, W_1 \in \operatorname{Aff}(\rrflex{d_i},\rrflex{d_{i+1}}),\, i=1,2
	\right\}$, and set 
\begin{equation}
	\circlearrowleft((W_{j,2},W_{j,1})_{j=1}^J)\triangleq W_{2,J}\circ \sigma \bullet W_{1,J}\circ \dots \circ W_{2,1}\circ \sigma \bullet W_{1,1}
	\label{eq_alala}
\end{equation}
	whenever the right-hand side of~\eqref{eq_alala} is well-defined.  Since the composition of two affine functions is again affine then $\NN$ is the set of deep feed-forward networks from $\rrm$ to $\rrn$ with activation function $\sigma$.  
\end{ex}
\begin{rremark}
The construction of Example~\ref{ex_fully_connected_DffNNS} parallels the formulation given in \cite{petersen2018topological,gribonval2019approximation}.  However, in \cite{gribonval2019approximation} elements of $\fff$ are referred to as neural networks and functions in $\NN$ are called their realizations.  
\end{rremark}
\begin{ex}[Trees]\label{ex_trees}
	Let $\xxx=L^1(\rr)$, $\fff\triangleq \{(a,b,c):a\in \rr, \, b,c\in \rr,\,b\leq c\}$, and let  $\circlearrowleft((a_j,b_j,c_j)_{j=1}^J)\triangleq \sum_{j=1}^J a_j I_{(b_j,c_j)}$.  Then, $\NN$ is the set of trees in $L^1(\rr)$.  
\end{ex}
We are interested in architectures which can generically approximate any function on their associated function space.  Paraphrasing \citep[page 67]{DeepLearningBook}, any such architecture is called a universal approximator.  
\begin{defn}[The Universal Approximation Property]\label{defn_UAP}
	An architecture $\arch$ is said to have the universal approximation property (UAP) if $\NN$ is dense in $\xxx$.  
\end{defn}
\vspace{-2em}
\section{Main Results}\label{s_1_charancterization}
Our first result provides a correspondence between the apriori algebraic structure of universal approximators on $\xxx$ and decompositions of $\xxx$ into subspaces on which $\NN$ contains the orbit of a topologically generic dynamical system, which are a priori of a topological nature.  The interchangeability of algebraic and geometric structures is a common theme, notable examples include \cite{GelfandDuality,IsbellDuality,StoneDualityGeneralizedDimov2012,SparseEffectiveNullstelesatz}.  
\begin{thrm}[Characterization: Dynamical Systems Structure of Universal Approximators]\label{thrm_Characterization_dynamics}
	Let $\xxx$ be a function space which is homeomorphic to an infinite-dimensional Fr\'{e}chet space and let $\arch$ be an architecture on $\xxx$.  Then, the following are equivalent:
	\begin{enumerate}[(i)]
		\item $\arch$ is a universal approximator,
		\item There exist subspaces $\{\xxx_i\}_{i \in I}$ of $\xxx$, continuous functions $\{\phi_i\}_{i \in I}$ with $\phi_i:\xxx_i\rightarrow \xxx_i$, and $\{g_i\}_{i \in I}\subseteq \NN$ such that:
		\begin{enumerate}[(a)]
			\item $\bigcup_{i \in I} \xxx_i$ is dense in $\xxx$,
			\item For each $i\in I$ and every pair of non-empty open $U,V\subseteq \xxx_i$, there is some $N_{i,U,V}\in \nn$ satisfying $$\phi^{N_{i,U,V}}(U%
			)\cap (V%
			) \neq \emptyset,
			$$
			\item For every $i \in I$, $g_i \in \xxx_i$ and $\{\phi_i^n(g_i)\}_{n \in \nn}$ is a dense subset of $\NN\cap \xxx_i$,
			\item For each $i \in I$, $\xxx_i$ is homeomorphic to $C(\rr)$.
		\end{enumerate}
			In particular, $\left\{\phi_i^n(g_i):\, i \in I,\, n \in \nn\right\}$ is dense in $\NN$.  
	\end{enumerate}
\end{thrm}
Theorem~\ref{thrm_Characterization_dynamics} describes the structure of universal approximators, however, it does not describe an explicit means of constructing them.  Nevertheless, Theorem~\ref{thrm_Characterization_dynamics} (ii.a) and (ii.d) suggest that universal approximators on most function spaces can be built by combining multiple, non-trivial, transformations of universal approximators on $C(\rrm,\rrn)$.  

This is type of transformation approach to architecture construction is common in geometric deep learning, whereby non-Euclidean data is mapped to the input of familiar architectures defined between $\rrd$ and $\rrD$ using specific feature maps and that model's outputs are then return to the manifold by inverting the feature map.  Examples include the hyperbolic feed-forward architecture of \cite{ganea2018hyperbolic}, and the shape space regressors of \cite{fletcher2011geodesic}, and the matrix-valued regressors of \cite{BonnabelRegression,baes2019lowrank}, amongst others.  This transformation procedure is a particular instance of the following general construction method, which extends \cite{kratsios2020non}.  
\begin{thrm}[Construction: Universal Approximators by Transformation]\label{thrm_Construction_Theorem}
	Let $n,m, \in \nn^+$, $\xxx$ be a function space, $\arch$ be a universal approximator on $C(\rrm,\rrn)$, and $\{\Phi_i\}_{i \in I}$ be a non-empty set of continuous functions from $C(\rrm,\rrn)$ to $\xxx$ satisfying the following condition:
	\begin{equation}
	\bigcup_{i \in I} \Phi_{i}\left(
	C(\rrm,\rrn)
	\right) \mbox{ is dense in } \xxx
	\label{eq_density_condition}
	.
	\end{equation} 
	Then $\arch[\Phi]$ has the UAP on $\xxx$, where $\fff_{\Phi}\triangleq \fff\times I$ and $
	\circlearrowleft_{\Phi}\left(
	\{f_j,i_j\}_{j=1}^J
	\right)\triangleq 
	\Phi_{I_J}\left(
	\circlearrowleft\left(
	(f_j)_{j=1}^J
	\right)
	\right)
	.
	$
\end{thrm}
The alternative approach to architecture development, subscribed to by authors such as \cite{hummel1992dynamic,bishop1994mixture,kipf2016semi,scarselli2008graph}, specifies the elementary functions $\fff$ and the rule for combining them.  Thus, this method explicitly specifies $\fff$ and implicitly specifies $\circlearrowleft$.  These competing approaches are in-fact equivalent since every universal approximator an approximately a transformation of the feed-forward architecture on $C(\rr)$.  
\begin{thrm}[Representation: Universal Approximators are Transformed Neural Networks]\label{thrm_Meta_Universal}
	Let $\sigma$ be a continuous, non-polynomial activation function, and let $\arch[0]$ denote the architecture of Example~\ref{ex_fully_connected_DffNNS}.  Let $\xxx$ be a function space which is homeomorphic to an infinite-dimensional Fr\'{e}chet.  If $\arch$ has the UAP on $\xxx$ then, there exists a family $\{\Phi_i\}_{i \in I}$ of embeddings $\Phi_i:C(\rr)\rightarrow \xxx$ such that for every $\epsilon>0$, $f \in \NN$ there exists some $i \in I$, $g_{\epsilon}\in \NN[\arch[0]]$, and $f_{\epsilon} \in \NN$ satisfying
	$$
	d_{\xxx}\left(
	f,\Phi_i(g_{\epsilon})
	\right) <\epsilon 
	\mbox{ and }
	d_{ucc}\left(
	g_{\epsilon}
	,
	\Phi_i^{-1}(f_{\epsilon})
	\right)
	<\epsilon
	.
	$$
\end{thrm}

The previous two results describe the structure of universal approximators but they do not imply the existence of such architectures.  Indeed, the existence of a universal approximator on $\xxx$ can always be obtained by setting $\fff=\xxx$ and $\circlearrowleft(f)=f$; however, this is uninteresting since $\fff$ is large, $\circlearrowleft$ is trivial, and $\NN$ is intractable.  Instead, the next result shows that, for a broad range of function spaces, there are universal approximators for which $\fff$ is a singleton, and the structure of $\circlearrowleft$ is parameterized by any prespecified separable metric space.  This description is possible by appealing to the free-space on $\xxx$.  

\begin{thrm}[Existence: Small Universal Approximators]\label{thrm_Existence}
	Let $X$ be a separable pointed metric space with at least two points, let $\xxx$ be a function space and a pointed metric space, and let $\xxx_0$ be a dense barycentric sub-space of $\xxx$.  Then, there exists a non-empty set $I$ with pre-order $\leq$, $\{x_i\}_{i \in I}\subseteq X-\{0_X\}$ there exist triples $\{(B_i,\Phi_i,\phi_i)\}_{i \in I}$ of linear subspaces $B_i$ of $B(\xxx_0)$, bounded linear isomorphisms $\Phi_i:B(X)\rightarrow B_i$, and bounded linear maps $\phi_i:B(X)\rightarrow B(X)$ satisfying:
	\begin{enumerate}[(i)]
		\item $B(\xxx_0)= \bigcup_{i \in I} B_i$,
		\item For every $i\leq j$, $B_i\subseteq B_j$,
		\item For every $i \in I$, $\bigcup_{n \in \nn^+} \Phi_i\circ \phi^n_i(x_i)$ is dense in $B_i$ with respect to its subspace topology,
		\item The architecture
		$\fff = \{x_i\}_{i \in I} 
		,
		$
		and 
		$
		\circlearrowleft|_{\fff^J}:(x_1,\dots,x_J)\triangleq \rho \circ \Phi_i\circ \phi_i^J\circ \delta_{x_j}
		$, whenever $x_1=x_j$ for each $j\leq J$, is a universal approximator on $\xxx$.  
	\end{enumerate}
	Furthermore, if $X=\xxx$ then the set $I$ is a singleton and $\Phi_i$ is the identity on $B(\xxx_0)$.  
\end{thrm}
The rest of this paper is devoted to the concrete implications of these results in learning theory.  
\section{Applications}\label{s_Implications}
The dynamical systems described by Theorem~\ref{thrm_Characterization_dynamics} (ii) can, in general, be complicated.  However, when $\arch$ is the feed-forward architecture with certain specific activation functions then these dynamical systems explicitly describe the addition of deep layers to a shallow feed-forward network.  We begin the next section by characterizing those activation function before outlining their approximation properties.  
\subsection{Depth as a Transitive Dynamical System}\label{ss_activ_gaps}
The impact of different activation functions on the expressiveness of neural network architectures is an active research area.  For example, \cite{Swishramachandran2017searching} empirically studies the effect of different activation function on expressiveness and in \cite{pinkus1999approximation} a characterization of the activation functions for which shallow feed-forward networks are universal is also obtained.  The next result characterizes the activation functions which produce feed-forward networks with the UAP even when no weight or bias is trained and the matrices $\{A_n\}_{n=1}^N$ are sparse, and the final layers of the network are slightly perturbed.  

Fix an activation function $\sigma:\rr\to\rr$.  For every $m\times m$ matrix $A$ and $b \in \rrm$, define the \textit{associated composition operator} $\Phi_{A,b}:f\mapsto f \circ \sigma \bullet(A\cdot +b)$, with terminology rooted in \cite{koopman1931hamiltonian}.  The family of composition operators $\{\Phi_{A,b}\}_{A,b}$ creates depth within an architecture $\arch$ by extending it to include any function of the form
$
\Phi_{A_N,b_N}\circ \dots \circ \Phi_{A_1,b_1}\left(\circlearrowleft((f_j)_{j=1}^J)\right),
$
for some $m\times m$ matrices $\{A_n\}_{n=1}^N$, $\{b_n\}$ in $\rrm$, and each $f_j \in \fff$ for $j=1,\dots,J$.  In fact, many of the results only require the following smaller extension of $\arch$, denoted by $\arch[{deep;\sigma}]$, where $\fff_{deep;\sigma}\triangleq \fff\times \nn$ and where
$$
\circlearrowleft_{deep;\sigma}\left(\{(f_j,n_j)\}_{j=1}^J\right)\triangleq \Phi^{N_J}_{I_m,b}\left(
\circlearrowleft((f_j)_{j=1}^J)
\right),
$$
and $b$ is any fixed element of $\rrm$ with positive components and $I_m$ is the $m\times m$ identity matrix.  
\begin{thrm}[Characterization of Transitivity in Deep Feed-Forward Networks]\label{thrm_activation_gaps}
	Let $\arch$ be an architecture on $C(\rrm,\rrn)$, $\sigma$ be a continuous activation function, fix any $b\in \rrm$ with strictly positive components.  Then $\Phi_{I_m,b}$ is a well-defined continuous linear map from $C(\rrm,\rrn)$ to itself and the following are equivalent:
	\begin{enumerate}[(i)]
		\item $\sigma$ is injective and has no fixed-points,
		\item Either $\sigma(x)>x$ or $\sigma(x)<x$ holds for every $x \in \rr$
		\item For every $g \in \arch$ and every $\delta>0$, there exists some $\tilde{g}\in C(\rrm,\rrn)$ with
		$
		d_{ucc}(g,\tilde{g})<\delta
		$
		such that, for each $f \in C(\rrm,\rrn)$ and each $\epsilon>0$ there is a $N_{g,f,\epsilon,\delta}\in \nn$ satisfying
		$$
		d_{ucc}(f,\Phi^{N_{g,f,\epsilon,\delta}}_{I_m,b}(\tilde{g}))<\epsilon,
		$$
		\item For each $\delta,\epsilon>0$ and every $f,g\in C(\rrm,\rrn)$ there is some $N_{U,V}\in \nn^+$ such that
		$$
		\left\{\Phi^{N_{\epsilon,\delta,g,f}}_{I_m,b}(\tilde{g}): \, d_{ucc}(\tilde{g},g)<\delta\right\} \cap \left\{
		\tilde{f}: \, d_{ucc}(\tilde{f},f)<\epsilon
		\right\} \neq \emptyset
		.
		$$
	\end{enumerate}
\end{thrm}
\begin{rremark}
	A characterization is given in Appendix~\ref{Appendix_B_Proof_of_applications_to_main_results} when $A\neq I_m$, however, this less technical formulation is sufficient for all our applications.  
\end{rremark}
We call an activation function \textit{transitive} if it satisfies any of the conditions (i)-(ii) in Theorem~\ref{thrm_activation_gaps}.  
\begin{ex}\label{ex_dynamical_properties_of_depth}
	The ReLU activation function $\sigma(x)=\max\{0,x\}$ does not satisfy Theorem~\ref{thrm_activation_gaps} (i).  
\end{ex}
\begin{ex}\label{ex_Leaky_ReLU_variant_A}
	The following variant of the Leaky-ReLU activation of \cite{maas2013rectifier} does satisfy Theorem~\ref{thrm_activation_gaps} (i)
	$$
	\sigma(x)
	\triangleq 
	\begin{cases}
	1.1 x + .1 & \, x \geq 0\\
	0.1 x + .1 & \, x < 0.
	\end{cases}
	$$
\end{ex}
More generally, transitive activation functions also satisfying the conditions required by the central results of \cite{pinkus1999approximation,kidger2019universal} can be build via the following.  
\begin{prop}[Construction of Transitive Activation Functions]\label{prop_examples_of_good_activations}
	Let $\tilde{\sigma}:\rr\to \rr$ be a continuous and strictly increasing function satisfying $\tilde{\sigma}(0)=0$.  Fix hyper-parameters $0<\alpha_1<1$, $0< \alpha_2$ such that $\alpha_2\neq \tilde{\sigma}'(0)-1$, and define
	$$
	\sigma(x)\triangleq \begin{cases}
	\tilde{\sigma}(x) + x + \alpha_2 &: \, x\geq 0\\
	\alpha_1 x + \alpha_2 &: x<0.
	\end{cases}
	$$	
	Then, $\sigma$ is continuous, injective, has no fixed-points, is non-polynomial, and is continuously differentiable with non-zero derivative on infinitely many points.  In particular, $\sigma$ satisfies the requirements of Theorem~\ref{thrm_activation_gaps}.
\end{prop}
\begin{rremark}
	\hspace{-.5em}
Any $\sigma$ built by Proposition~\ref{prop_examples_of_good_activations} meets the conditions of \citep[Theorem 3.2]{kidger2019universal} and \citep[Theorem 1]{leshno1993multilayer}.
\end{rremark}
Transitive activation functions allow one to automatically conclude that $\arch[\sigma;\deep]$ has the UAP on $C(\rrm,\rrn)$ if $\arch$ is only a universal approximator on some non-empty open subset thereof.  
\begin{cor}[Local-to-Global UAP]\label{cor_approx_loc_to_global}
	Let $\xxx$ be a non-empty open subset of $C(\rrm,\rrn)$ and $\arch$ be a universal approximator on $\xxx$.  If any of the conditions described by Lemma~\ref{lem_activation_gaps_Birkohoff} (i)-(iii) hold, then $\arch[\sigma;\deep]$ is a universal approximator on $C(\rrm,\rrn)$. 
\end{cor}
The function space affects which activation functions are transitive.  Since most universal approximation results hold in the space $C(\rrm,\rrn)$ or on $L^p_{\mu}(\rrm)$, for suitable $\mu$ and $p$, we describe the integrable variant of transitive activation functions.  
\subsubsection{Integrable Variants}\label{ss_Integrable_Variants}
Some notation is required when expressing the integrable variants of the Theorem~\ref{thrm_activation_gaps} and its consequences.  
Fix a $\sigma$-finite Borel measure $\mu$ on $\rrm$.  Unlike in the continuous case, the operators $\Phi_{A,b}$ may not be well-defined or continuous from $L^1_{\mu}(\rrm)$ to itself.  We require the notion of a push-forward measure by a measurable function is required.  If $S:\rrm\rightarrow \rrm$ is Borel measurable and $\mu$ is a finite Borel measure on $\rrm$, then its push-forward by $S$ is the measure denoted by $S_{\#}\mu$ and defined on Borel subsets $B\subseteq \rrm$ by
$
S_{\#}\mu(B)\triangleq \mu\left(S^{-1}[B]\right).
$
In particular, if $\mu$ is absolutely continuous with respect to the Lebesgue measure $\mu_M$ on $\rrm$, then as discussed in \citep[Chapter 2.1]{SinghManhasCompositionOpsFunSpaces1993}, $S_{\#}\mu$ admits a Radon-Nikodym derivative with respect to the Lebesgue measure on $\rrm$.  We denote this Radon-Nikodym derivative by $\frac{dS_{\#}\mu}{d\mu_M}$.  
A finite Borel measure $\mu$ on $\rrm$ is equivalent to the Lebesgue measure thereon, denoted by $\mu_M$ if both $\mu_M$ and $\mu$ are absolutely continuous with one another.  

Recall that, if a function is monotone on $\rr$, then it is differentiable outside a $\mu_M$-null set.  We denote the $\mu_M$-a.e. derivative of any such a function $\sigma$ by $\sigma'$.  Lastly, we denote the essential supremum of any $f \in L^1_{\mu}(\rrm)$ by $\|f\|_{L^{\infty}}$.  The following Lemma is a rephrasing of \citep[Corollary 2.1.2, Example 2.17]{SinghManhasCompositionOpsFunSpaces1993}.
\begin{lem}\label{lem_composition_Lp_basic_welldefinedness_and_norm_computation}
	Fix a $\sigma$-finite Borel measure $\mu$ on $\rrm$ equivalent to the Lebesgue measure, let 
	$1\leq p<\infty$, 
	$b \in \rrm$, $A$ be an $m\times m$ matrix, and	let $\sigma:\rr\rightarrow \rr$ be a Borel measurable.  
	Then, the composition operator $\Phi_{A,b}:L^1(\rrm;\rrn)\rightarrow L^1(\rrm;\rrn)$ is well-defined and continuous if and only if $
	(\sigma\bullet (A \cdot +b))_{\#}\mu$ is absolutely-continuous with respect to $\mu$ and 
	\begin{equation}
	\left\|
	\frac{d(\sigma\bullet (A \cdot +b))_{\#}\mu}{d\mu_M}
	\right\|_{L^{\infty}}<\infty
	.
	\label{eq_well_defindeness_condition}
	\end{equation}
	In particular, when $\sigma$ is monotone then $\Phi_{I_m,b}$ is well-defined if and only if there exists some $M>0$ such that for every $x \in \rr$, 
	$
	M\leq \sigma'(x+b).
	$
\end{lem}
For $g \in L^1_{\mu}(\rrm,\rrn)$ and $\delta>0$, we denote the set of all functions $f \in L^1_{\mu}(\rrm,\rrn)$ satisfying $\int_{x \in \rr} \|f(x)-g(x)\|d\mu(x)<\epsilon$ by $\operatorname{Ball}_{L^1_{\mu}(\rrm,\rrn)}(g,\delta)$.  A function is called \textit{Borel bi-measurable} if both the image and pre-images of Borel sets, under that map, are again Borel sets.  
\begin{cor}[Transitive Activation Functions (Integrable Variant)]\label{cor_Composition_Lp}
	Let $\mu$ be a $\sigma$-finite measure on $\rrm$, let $b \in \rrm$ with $b_i>0$ for $i=1,\dots,m$, and suppose that $\sigma$ is injective, Borel bi-measurable, that $\sigma(x)>x$ except on a Borel set of $\mu$-measure $0$, and assume that condition~\eqref{eq_well_defindeness_condition} holds.  If $\arch$ has the UAP on $\operatorname{Ball}(g,\delta)$ for some $f\in L^1_{\mu}(\rrm)$ and some $\delta>0$ then, for every $f \in L^1_{\mu}(\rrm)$ and every $\epsilon>0$ there exists some $f_{\epsilon}\in \NN$ and $N_{\epsilon,\delta,f,g} \in \nn$ such that
	$$
	\int_{x \in \rrm} \left\|f(x)- \Phi^{N_{\epsilon,\delta,f,g}}_{I_m,b}\left(f_{\epsilon}(x)\right)\right\|d\mu(x)<\epsilon
	.
	$$
\end{cor}
We call activation functions satisfying the conditions of Corollary~\ref{cor_Composition_Lp} \textit{$L^p_{\mu}$-transitive.  }
The following is a sufficiency condition analogous to the characterization of Proposition~\ref{prop_examples_of_good_activations}.  
\begin{cor}[Construction of Transitive Activation Functions (Integrable Variant)]\label{cor_examples_of_good_activations_Lp}
	Let $\mu$ be a finite Borel measure on $\rrm$ which is equivalent to $\mu_M$.  
	Let $\tilde{\sigma}:[0,\infty)\to [0,\infty)$ be a surjective continuous and strictly increasing function satisfying $\tilde{\sigma}(0)=0$, let $0<\alpha_1<1$.  Define the activation function
	$$
	\sigma(x)\triangleq \begin{cases}
	\tilde{\sigma}(x) + x & : \, x\geq 0\\
	\alpha x &:\, x <0.
	\end{cases}
	$$
	Then $\sigma$ is Borel bi-measurable, $\sigma(x)>x$ outside a $\mu_M$-null-set, it is non-polynomial, and it is continuously differentiable with non-zero derivative for every $x<0$.  
\end{cor}
Different function spaces can have different transitive activation functions.  By shifting the Leaky-ReLU variant of Example~\ref{ex_Leaky_ReLU_variant_A} we obtain an $L^p$-transitive activation function which fails to be transitive.  
\begin{ex}[Rescaled Leaky-ReLU is $L^p$-Transitive]\label{ex_modified_Leaky_ReLU_revisited}
	The following variant of the Leaky-ReLU activation function
	$$
	\sigma(x)
	\triangleq 
	\begin{cases}
	1.1 x & \, x \geq 0\\
	0.1 x & \, x < 0,
	\end{cases}
	$$
	is a continuous bijection on $\rr$ with continuous inverse and therefore it is injective and bi-measurable.  Since $0$ is its only fixed point, then the set $\{\sigma(x)\not >x\}=\{0\}$ is of Lebesgue measure $0$, and thus of $\mu$ measure $0$ since $\mu$ and $\mu_M$ are equivalent.  Hence, $\sigma$ is injective, Borel bi-measurable, that $\sigma(x)>x$ except on a Borel set of $\mu$-measure $0$, as required in~\eqref{cor_Composition_Lp}.  However, since $0$ is a fixed point of $\sigma$ then it does not meet the requirements of Theorem~\ref{thrm_activation_gaps} (i).  
\end{ex}
Our main interest with transitive activation functions is that they allow for refinements of classical universal approximation theorems, where a network's last few layers satisfy constraints.  This is interesting since constraints are common in most practical citations.  
\subsection{Deep Networks with Constrained Final Layers}\label{ss_U_i_DTL}
The requirement that the final few layers of a neural network to resemble the given function $\hat{f}$ is in effect a constraint on the network's output possibilities.  The next result shows that, if a transitive activation function is used, then a deep feed-forward network's output layers may always be forced to approximately behave like $\hat{f}$ while maintaining that architecture's universal approximation property.  Moreover, the result holds even when the network's initial layers are sparsely connected and have breadth less than the requirements of \cite{park2020minimum,kidger2019universal}.  Note that, the network's final layers must be fully connected and are still required to satisfy the width constraints of \cite{kidger2019universal}.  For a matrix $A$ (resp. vector $b$) the quantity $\|A\|_0$ (resp. $\|b\|_0$) denotes the number of non-zero entries in $A$ (resp. $b$).  
\begin{cor}[Feed-Forward Networks with Approximately Prescribed Output Behavior]\label{cor_biased_nets}
	Let $\hat{f}:\rrm\rightarrow \rrn$, $\epsilon,\delta>0$, 
	and let $\sigma$ be a transitive activation function which is non-affine continuous and differentiable at-least at one point with non-zero derivative at that point.  If there exists a continuous function $\tilde{f}_0:\rrm\rightarrow \rrn$ such that 
	\begin{equation}
	d_{ucc}(f_0,\tilde{f}_0)<\delta
	\label{eq_regularitycondition}
	,
	\end{equation}
	then there exists $f_{\epsilon,\delta}\in \NN$, $J,J_1,J_2 \in \nn^+$, $0\leq J_1<J$, and sets of composable affine maps $\{W_j\}_{j=1}^J$, $\{\tilde{W}_j\}_{j=1}^{J_2}$ such that $
	f_{\epsilon,\delta}=W_J\circ \sigma\bullet \dots \circ \sigma \bullet W_1
	$ and the following hold:
	\begin{enumerate}[(i)]
		\item $		d_{ucc}\left(\hat{f},
		W_J\circ \sigma \bullet\dots \circ \sigma\bullet W_{J_1}
		\right) <\delta,$
		\item $d_{ucc}\left(f,f_{\epsilon,\delta}\right)<\epsilon$,
		\item $\max_{j=1,\dots,J_1} \|A^{W_j}\|_0\leq m$,
		\item $W_j:\rrflex{d_j}\rightarrow \rrflex{d_{j+1}}$ is such that $d_j\leq m+n+2$ if $J_1<j\leq J$ and $d_j=m$ if $0\leq j\leq J_1$.  
	\end{enumerate}
	If $J_1=0$ we make the convention that $W_{J_1}\circ \sigma \bullet\dots \circ \sigma \bullet W_1(x)=x$.  
\end{cor}
\begin{rremark}\label{remark_explanation_of_condition}
Condition~\ref{eq_regularitycondition}, for any $\delta>0$, whenever $f_0$ is continuous.  
\end{rremark}
We consider an application of Corollary~\ref{cor_biased_nets} to deep transfer learning.  As described in \cite{bengio2012deep}, deep transfer learning is the practice of transferring knowledge from a pre-trained model into a neural network architecture which is to be trained on a, possibly new, learning task.  Various formalizations of this paradigm are described in \cite{tan2018survey} and the next example illustrates the commonly used approach, as outlined in \cite{chollet2015keras}, where one first learns a feed-forward network $\hat{f}:\rrm\rightarrow \rrn$ and then uses this map to initialize the final portion of a deep feed-forward network.  Here, given a neural network $\hat{f}$, typically trained on a different learning task, we seek to find a deep feed-forward network whose final layers are arbitrarily close to $\hat{f}$ while simultaneously providing an arbitrarily precise approximation to a new learning task.  
\begin{ex}[Feed-Forward Networks with Pre-Trained Final Layers are Universal]\label{ex_DTL_feed_Forward}
	Fix a continuous activation function $\sigma$, let $N>0$ be given, let $\arch$ as in Example~\ref{ex_fully_connected_DffNNS}, let $K$ be a non-empty compact subset of $\rrm$, and let $\hat{f}\in \NN$.  Corollary~\ref{cor_biased_nets} guarantees that there is a deep feed-forward neural network $f_{\epsilon,\delta}=W_J\circ \sigma\bullet \dots \circ \sigma \bullet W_1$ satisfying
	\begin{enumerate}[(i)]
		\item $\sup_{x \in K}\left\|\hat{f}(x)
		-W_J\circ \sigma \bullet\dots \circ \sigma\bullet W_{J_1}(x)\right\| <N^{-1},$
		\item $\sup_{x \in K}\left\|f (x)- f_{\epsilon,\delta}(x) \right\|<N^{-1}$,
		\item $\max_{j=1,\dots,J_1} \|A^{W_j}\|_0\leq m$,
		\item $W_j:\rrflex{d_j}\rightarrow \rrflex{d_{j+1}}$ is such that $d_j\leq m+n+2$ if $J_1<j\leq J$ and $d_j=m$ if $0\leq j\leq J_1$.  
	\end{enumerate}
\end{ex}
The structure imposed on the architecture's final layers can also be imposed by a set of constraints.  The next result shows that, for a feed-forward network with a transitive activation function, the architecture's output can always be made to satisfy a finite number of compatible constraints.  These constraints are described by a finite set of continuous functionals $\{F_n\}_{n=1}^N$ on $C(\rrm,\rrn)$ together with a set of thresholds $\{C_n\}_{n=1}^N$, where each $C_n >0$.  
\begin{cor}[Feed-Forward Networks with Constrained Final Layers are Universal]\label{cor_Constrained_archs_are_Univ}
	Let $\sigma$ be a transitive activation function which is non-affine continuous and differentiable at-least at one point with non-zero derivative at that point, let $\arch$ denote the feed-forward architecture of Example~\ref{ex_fully_connected_DffNNS}, $\{F_n\}_{n=1}^N$ be a set of continuous functions from $C(\rrm,\rrn)$ to $[0,\infty)$, and $\{C_n\}_{n=1}^N$ be a set of positive real numbers.  If there exists some $f_0\in C(\rrm,\rrn)$ such that for each $n=1,\dots,N$ the following holds
	\begin{equation}
	F_n(f_0)< C_n
	\label{eq_constraint_compatability}
	,
	\end{equation}
	then for every $f \in C(\rrm,\rrn)$ and every $\epsilon>0$, there exist $f_{1,\epsilon},f_{2,\epsilon}\in \NN$, diagonal $m\times m$-matrices $\{A_j\}_{j=1}^J$ and $b_1,\dots,b_J \in \rrm$ satisfying:
	\begin{enumerate}[(i)]
		\item $f_{2,\epsilon}\circ f_{1,\epsilon}$ is well-defined,
		\item $d_{ucc}\left(f,f_{2,\epsilon}\circ f_{1,\epsilon}\right)<\epsilon$,
		\item $f_{2,\epsilon} \in \bigcap_{n=1}^N F^{-1}_n\left[[0,C_n)\right]$, 
		\item $f_{1,\epsilon}(x)=  \sigma \bullet(A_n\cdot + b_n) \circ \dots \circ \sigma \bullet(A_1x + b_1)$.
	\end{enumerate} 
\end{cor}
Next, we show that transitive activation functions can be used to extend the currently-available approximation rates for shallow feed-forward networks to their deep counterparts.  
\subsection{Approximation Bounds for Networks with Transitive Activation Function}\label{ss_Approximation_Bounds}
In \cite{barron1993universal,darken1993rate}, it is shown that the set of feed-forward neural networks of breadth $N\in \nn^+$, can approximate any function lying in their closed convex hull of at a rate of $\mathscr{O}(N^{\frac{-1}{2}})$.  These results do not incorporate the impact of depth into its estimates and the next result builds on them by incorporating that effect.  
As in \cite{darken1993rate}, the convex-hull of a subset $A\subseteq L^1_{\mu}(\rrm)$ is the set $\co{A}\triangleq \left\{\sum_{i=1}^n \alpha_i f_i:\, f_i \in A,\, \alpha_i \in [0,1],\, \sum_{i=1}^n \alpha_i=1\right\}$ and the interior of $\co{A}$, denoted $\operatorname{int}(\co{A})$, is the largest open subset thereof.  
\begin{cor}[Approximation-Bounds for Deep Networks]\label{cor_Explanation_of_Depth}
	Let $\mu$ be a finite Borel measure on $\rrm$ which is equivalent to the Lebesgue measure, $\fff\subseteq L^1_{\mu}(\rrm)$ for which $\operatorname{int}(\co{\fff})$ is non-empty and $\co{\fff}\cap \operatorname{int}(\co{\fff})$ is dense therein.  If $\sigma$ is a continuous non-polynomial $L^1$-transitive activation function, $b \in \rrm$ have positive entries, and that~\eqref{eq_well_defindeness_condition} is satisfied, then the following hold:
	\begin{enumerate}[(i)]
		\item For each $f \in L^1_{\mu}(\rrm)$ and every $n \in \nn$, there is some $N \in \nn$ such that the following bound holds
		$$
		\begin{aligned}
		\inf_{f_i \in \fff,\, \sum_{i=1}^n \alpha_i=1,\, \alpha_i \in [0,1]} &
		\int_{x \in \rrm} \left\|
		\sum_{i=1}^n \alpha_i 
		\Phi_{I_m,b}^N\left(
		f_i\right)
		(x)
		-
		f(x)
		\right\|d\mu(x)
		\leq
		\frac{
			\left\|
			\frac{d(\sigma\bullet(\cdot + b))_{\#}\mu}{d\mu_M}
			\right\|_{\infty}^{\frac{N}{2}}	
		}{\sqrt{n}}\left(
		1 + \sqrt{2\mu(\rrm)}
		\right)
		.
		\end{aligned}
		,
		$$
		\item There exists some $\kappa>1$ such that $\left\|
		\frac{d(\sigma \bullet(\cdot + b)_{\#}\mu}{d\mu_M}
		\right\|_{\infty} > \kappa^N$.  In particular, $\lim\limits_{N \to \infty} \left\|
		\frac{d(\sigma \bullet(\cdot + b))_{\#}\mu}{d\mu_M}
		\right\|_{\infty}^{\frac{N}{p}} = \infty$,
		\item $\left\{
		\sum_{i=1}^n \alpha_i \Phi_{I_m,b}^N(f_i):\, N,n \in \nn,\, f_i \in \fff,\, \alpha_i \in [0,1],\, \sum_{i=1}^n \alpha_i =1
		\right\}$ is dense in $L^1_{\mu}(\rrm)$.
	\end{enumerate}
\end{cor}
\begin{rremark}
	Unlike in \cite{darken1993rate}, Corollary~\ref{cor_Explanation_of_Depth} (i) holds even when the function $f$ does not lie in the closure of $\co{\fff}$.  This is entirely due to the topological transitivity of the composition operator $\Phi_{I_m,b}$ and is therefore entirely due to the depth present in the network.  In particular, Corollary~\ref{cor_Explanation_of_Depth} (iii) implies that universal approximation can be achieved even if a feed-forward networks' output weights are all constrained to satisfy $\sum_{i=1}^n \alpha_i=1$ and $\alpha_i=[0,1]$ and even if all but the architecture's final two layers are sparsely connected and not trainable.  
\end{rremark}
To date, we have focused on the application and interpretation of Theorem~\ref{thrm_Characterization_dynamics}.  Next,  Theorem~\ref{thrm_Meta_Universal} is used to modify and improve the approximation capabilities of universal approximators on $C(\rr)$.  
\subsection{Improving the Approximation Capabilities of an Architecture}\label{ss_refinement}
Most currently available universal approximation results for spaces of continuous functions, provide approximation guarantees for the topology of uniform convergence on compacts.  Unfortunately, this is a very local form of approximation and there is no guarantee that the approximation quality holds outside a prespecified bounded set.  For example, the sequence $f_n(x)\triangleq e^{-\frac{1}{1-(x-n)^2}}I_{|x-n|\leq 1}$ converges to the constant $0$ function, uniformly on compacts while maintaining the constant error $\sup_{x \in \rr} \|f_n(x)-0\|=1$.  

These approximation guarantees are strengthened by modifying any given universal approximator on $C(\rrm,\rrn)$ to obtain a universal approximator in a smaller space of continuous functions for a much finer topology.  We introduce this space as follows.  

Let $\Omega$ be a finite set of non-negative-valued, continuous functions $\omega$ from $[0,\infty)$ to $[0,\infty)$ for which there is some $\omega_0 \in \Omega$ satisfying $\omega_0(\cdot)=1$.  Let $C_{\Omega}(\rrm,\rrn)$ be the set of all continuous functions whose asymptotic growth-rate is controlled by some $\omega\in \Omega$, in the sense that, $C_{\Omega}(\rrm,\rrn)\triangleq \bigcup_{\omega \in \Omega} C_{\omega}(\rrm,\rrn)$, where $f \in C_{\omega}(\rrm,\rrn)$ if $\|f\|_{\omega,\infty}\triangleq \frac{\|f(x)\|}{\omega(\|x\|)+1}<\infty.$  Each $C_{\omega}(\rrm,\rrn)$ is a special case of the weighted spaces studied in \cite{ProllaWeightedSpaces}, which are Banach spaces when equipped with the norm $\|\cdot\|_{\omega,\infty}$.  Accordingly, $C_{\Omega}(\rrm,\rrn)$ is equipped with the finest topology making each $C_{\omega}(\rrm,\rrn)$ into a subspace.  Indeed, such a topology exists by \citep[Proposition 2.6]{BourbakiTopGen}.  
\begin{ex}\label{ex_imntuition}
	If $\Omega=\{\max\{t,t^{i}\}\}_{i >0}$ then $f \in C_{\Omega}(\rrm,\rrn)$ if and only if $f$ has asymptotically sub-polynomial growth, in the sense that, there is a polynomial $p:\rrm\rightarrow \rrn$ with $\lim\limits_{\|x\|\to \infty} \frac{\|f(x)\|}{(\|p(x)\|+1)}<\infty$.  
\end{ex}
Given an architecture $\arch$ on $C(\rrm,\rrn)$, define its $\Omega$-modification to be the architecture $\arch[\Omega]$ on $C_{\Omega}(\rrm,\rrn)$ given by $\fff_{\Omega}\triangleq \fff  \times \Omega \times (0,\infty)^2$ and where
$$
\begin{aligned}
\circlearrowleft\left(
\left\{f_j,\alpha_j,\omega_j,b_j,a_j\right\}_{j=1}^J
\right)\triangleq &\omega_J(\|\cdot \|+1)
\left[
\left(f_{} e^{-\frac{b_J}{b_J - \|\cdot\|^2}} + a_J\right)I_{\|\cdot\|< b_J}
+  \left(a_Je^{-
	\left|
	f_{}(\cdot)
	\right|(\|x\|-b_J)}\right)I_{\|\cdot\|\geq b_J}
\right]
,
\\
f\triangleq &\circlearrowleft(f_J,\dots, f_1)
\end{aligned}
$$
Therefore, the functions in $\NN[\arch[\Omega]]$ are capable of adjusting to the different growth rates of functions in $C_{\Omega}(\rrm,\rrn)$ into continuous functions of different growth rates; whereas those in $\arch$ need not be.  
\begin{thrm}[{$\arch[\Omega]$ is a Universal Approximator in $C_{\Omega}(\rrm,\rrn)$}]\label{thrm_sharpened_UAT}
	If $\arch$ is a universal approximator on $C(\rrm,\rrn)$ for which each $f \in \NN$ satisfies the following growth condition
\begin{equation}
	\sup_{x \in \rrm} \|f(x)\|e^{-\|x\|}<\infty,
	\label{eq_condition_C_growth}
\end{equation}
	then $\arch[\Omega]$ is a universal approximator on $C_{\Omega}(\rrm,\rrn)$.  
\end{thrm}
\begin{rremark}
Condition~\eqref{eq_condition_C_growth} is satisfied by any set of piecewise linear functions.  For instance, $\NN$ is comprised of piecewise linear functions if $\fff$ is as in Example~\ref{ex_fully_connected_DffNNS} and $\sigma$ is the ReLU activation function.  
\end{rremark}
The architecture $\arch[\Omega]$ often provides a strict improvement over $\arch$.  
\begin{prop}\label{prop_limitation_PWL_UAs}
	Let $\arch$ be a universal approximator on $C(\rrm,\rrn)$ such that each $f \in \NN$ is either constant or $\sup_{x \in \rrm} \|f(x)\|=\infty$, and let $\Omega\triangleq \{\exp(-k t): n \in \nn\}$.  
	Then $\arch$ is not a universal approximator on $C_{\Omega}(\rrm,\rrn)$.  
\end{prop}
\subsection{{Representation of Approximators on $L^{\infty}(\rr)$}}
There is currently no available universal approximation theorem describing a small architecture on $L^{\infty}(\rrm,\rrn)$ with the UAP.  Indeed, even trees are not dense therein since the Lebesgue measures is $\sigma$-finite and not finite.  A direct consequence of Theorem~\ref{thrm_Existence} is the guarantee that a minimal architecture on $L^{\infty}(\rr)$ exists and admits the following representation.  
\begin{cor}[Existence and Representation of Minimal Universal Approximator on $L^{\infty}(\rr)$]\label{cor_Existence_Linfinity}
	There exists a non-empty set $I$ with pre-order $\leq$, a subset $\{x_i\}_{i \in I} \subseteq L^1(\rr)-\{0\}$, triples $\{(B_i,\Phi_i,\phi_i)\}_{i \in I}$ of linear subspaces $B_i$ of $B(L^{\infty})$, bounded linear isomorphisms $\Phi_i:L^1(\rr)\rightarrow B_i$, and bounded linear maps $\phi_i:L^1(\rr)\rightarrow L^{1}(\rr)$ such that:
	\begin{enumerate}[(i)]
		\item $B(L^{\infty})= \bigcup_{i \in I} B_i$,
		\item For every $i\leq j$, $B_i\subseteq B_j$,
		\item For every $i \in I$, $\bigcup_{n \in \nn^+} \Phi_i\circ \phi^n_i(x_i)$ is dense $B_i$ for its subspace topology,
		\item The architecture $\arch$ defined by
		\begin{equation}
		\begin{aligned}
		&\fff = \{x_i\}_{i \in I} 
		,
		&\,
		\circlearrowleft|_{\fff^J}:(x_1,\dots,x_j)\triangleq 
		\rho \circ \Phi_i\circ \phi_i^j%
		\circ \eta({x_i}) 
		,
		\end{aligned}
		\label{eq_ANSARI_BERNAL_Architecture_full}
		\end{equation}
		if $x_1=x_j$, for each $j\leq J$, has the UAP on $L^{\infty}(\rr)$, where $\eta:\rr\to L^1$ and $\rho:B(L^{\infty})\rightarrow L^{\infty}$ are respectively defined as the linear extensions of the maps
		$$
		\begin{aligned}
		&\eta(r)
		\triangleq & \begin{cases}
		I_{[0,r)} & : s> 0\\
		- I_{[-r,0)} & : s< 0
		,
		\end{cases}
		\qquad & 
		\rho\left(
		\sum_{i=1}^n \alpha_i \delta_{f_i}
		\right)\triangleq & \frac1{n}\sum_{i=1}^n \alpha_i f_i
		.
		\end{aligned}
		$$
	\end{enumerate}
\end{cor}
The contributions of this article are now summarized.  
\section{Conclusion}\label{s_Conclusion}
In this paper, we studied the universal approximation property in a scope applicable to most architectures on most function spaces of practical interest.  Our results were used to characterize, construct, and establish the existence of such structures both in many familiar and exotic function spaces.  

Our results were used to establish the universal approximation capabilities of deep and narrow networks with constraints on their final layers and sparsely connected initial layers.  We derived approximation bounds for feed-forward networks with this activation function in terms of depth and height.  We showed that the set of activation functions for which these results hold is broader when the underlying functions space is $L^p(\rrm)$ than if it is $C(\rrm)$, which showed that the choice of activation function depends on the underlying topological criterion quantifying the UAP.  We characterized the activation functions for which these results hold as precisely being the set of injective, continuous, non-affine activation functions which are differentiable at at-least one point with non-zero derivative at that point and have no fixed points. We provided a simple direct way to construct these activation functions.  We showed that a rescaled and shifted Leaky-ReLU activation is an example of such an activation function while the ReLU activation is not.  We used our construction result to build a universal approximator in the space of continuous functions between Euclidean spaces, which have controlled growth, equipped with a uniform notion of convergence.  This result strengthens the currently available guarantees for feed-forward networks, which state that this architecture is universal in $C(\rrm,\rrn)$ for the weaker uniform convergence on compacts topology.  Finally, we obtained a representation of a small universal approximator on $L^{\infty}(\rrm)$.  

The results, structures, and methods introduced in this paper provide a flexible and broad toolbox to the machine learning community to build, improve, and understand universal approximators.  It is hoped that these tools will help others develop new, theoretically justified architectures for their learning tasks.

\bibliographystyle{abbrvnat}
\bibliography{References_Final}	
\printindex

\begin{appendices}
	\section{Proofs of Main Results}\label{s_A_Proofs}
	Theorem~\ref{thrm_Characterization_dynamics} is encompassed by the following broader but more technical result.  
	\begin{lem}[Characterization of the Universal Approximation Property]\label{lemma_full_Characterization}
		Let $\xxx$ be a function space, $E$ is an infinite-dimensional Fr\'{e}chet space for which there exits some homeomorphism $\Phi:\xxx\rightarrow E$, and $\arch$ be an architecture on $\xxx$.  Then the following are equivalent:
		\begin{enumerate}[(i)]
			\item \textbf{UAP:} $\arch$ has the UAP,
			\item \textbf{Decomposition of UAP via Subspaces:} There exist subspaces $\{\xxx_i\}_{i \in I}$ of $\xxx$ such that:
			\begin{enumerate}[(a)]
				\item $\bigcup_{i \in I} \xxx_i$ is dense in $\xxx$,
				\item For each $i \in I$, $\Phi(\xxx_i)$ is a separable infinite-dimensional Fr\'{e}chet subspace of $E$ and $\Phi\left(\NN\cap \xxx_i\right)$ contains a countable, dense, and linearly-independent subset of $\Phi(\xxx_i)$,
				\item For each $i \in I$, there exists a homeomorphism $\Phi_i:\xxx_i \rightarrow L^2(\rr)$.
			\end{enumerate}
			\item \textbf{Decomposition of UAP via Topologically Transitive Dynamics:} There exist subspaces $\{\xxx_i\}_{i \in I}$ of $\xxx$ and continuous functions $\{\phi_i\}_{i \in I}$ with $\phi_i:\xxx_i\rightarrow \xxx_i$ such that:
			\begin{enumerate}[(a)]
				\item $\bigcup_{i \in I} \xxx_i$ is dense in $\xxx$,
				\item For every pair of non-empty open subsets $U,V$ of $\xxx$ and every $i \in I$, there is some $N_{i,U,V}\in \nn$ such that $\phi^{N_{i,U,V}}(U\cap \xxx_i)\cap (V\cap \xxx_i) \neq \emptyset$,
				\item For every $i \in I$, there is some $g_i \in \NN\cap \xxx_i$ such that $\{\phi_i^n(g_i)\}_{n \in \nn}$ is a dense subset of $\NN\cap \xxx_i$, and in particular, it is a dense subset of $\xxx_i$,
				\item For each $i \in I$, $\xxx_i$ is homeomorphic to $C(\rr)$.  
			\end{enumerate}
			\item \textbf{Parameterization of UAP on Subspaces:} There are triples $\{(X_i,\Phi_i,\psi_i)\}_{i\in I}$ of separable topological spaces $X_i$, non-constant continuous functions $\Phi_i:X_i\to \xxx$, and functions $\psi_i:X_i\rightarrow X_i$ satisfying the following:
			\begin{enumerate}[(a)]
				\item $\bigcup_{i \in I} \Phi_i(X_i)$ is dense in $\xxx$,
				\item For every $i \in I$ and every pair of non-empty open subsets $U,V$ of $X_i$, there is some $N_{i,U,V}\in \nn$ such that $\psi^{N_{i,U,V}}(U\cap X_i)\cap (V\cap X_i) \neq \emptyset$,
				\item For every $i \in I$, there is some $x_i \in \NN\cap X_i$ such that $\{\Phi_i\circ \psi_i^n(x_i)\}_{n \in \nn}$ is a dense subset of $\NN\cap \Phi_i(X_i)$, and in particular, it is a dense subset of $\Phi_i(X_i)$.
			\end{enumerate}
		\end{enumerate}
		Moreover, if $\xxx$ is separable, then $I$ may be taken to be a singleton.  
	\end{lem}
	\begin{proof}[{Proof of Lemma~\ref{lemma_full_Characterization}}]
		Suppose that (ii) holds.  Since $\bigcup_{i \in I} \xxx_i$ is dense in $\xxx$ and since $\bigcup_{i \in I} \NN\cap \xxx_i\subseteq \NN$, then, it is sufficient to show that $\bigcup_{i \in I} \NN \cap \xxx_i$ is dense in $\bigcup_{i \in I} \xxx_i$ to conclude that is is dense in $\xxx$.  Since each $\xxx_i$ is a subspace of $\xxx$ then, by restriction, each $\xxx_i$ is a subspace of $\bigcup_{i \in I} \NN \cap \xxx_i$ with its relative topology.  
		
		Let $\tilde{\xxx}$ denote the set $\bigcup_{i \in I} \xxx_i$ equipped with the finest topology making each $\xxx_i$ into a subspace, such a topology exists by \citep[Proposition 2.6]{BourbakiTopGen}.  Since each $\xxx_i$ is also a subspace of $\bigcup_{i \in I} \xxx_i$ with its relative topology and since, by definition, that topology is no finer than the topology of $\tilde{\xxx}$ then it is sufficient to show that $\bigcup_{i \in I} \NN \cap \xxx_i$ is dense in $\tilde{X}$ to conclude that it is dense in $\bigcup_{i \in I} \xxx_i$ equipped with its relative topology.  	
		
		Indeed, by \citep[Proposition 2.7]{BourbakiTopGen} the space $\tilde{X}$ is given by the (topological) quotient of the disjoint union $\sqcup_{i \in I} \xxx_i$, in the sense of topological spaces (see \citep[Example 3, Section 2.4]{BourbakiTopGen}), under the equivalence relation $f_i\sim f_j$ if $f_i=f_j$ in $\xxx$.  Denote the corresponding quotient map by $Q_{\tilde{\xxx}}$.  Since a subset $U$ of the quotient topology is open (see \citep[Example 2, Section 2.4]{BourbakiTopGen}) if and only if $Q_{\tilde{\xxx}}^{-1}[U]$ is an open subset of $\sqcup_{i \in I} \xxx_i$ and since a subset $V$ of $\sqcup_{i \in I} \xxx_i$ is open if and only if $V\cap \xxx_i$ is open for each $i \in I$ in the topology of $\xxx_i$ then $U\subseteq \tilde{\xxx}$ is open if and only if $Q_{\tilde{\xxx}}^{-1}[U] \cap \xxx_i$ is open for each $i \in I$.  Since $\{\NN\cap \xxx_i\}_{n \in \nn^+}$ is dense in $\xxx_i$ then for every open subset $U'\subseteq \xxx_i$
		\begin{equation}
		\emptyset \neq U' \cap 
		\NN\cap \xxx_i 
		\subseteq
		U' \cap \bigcup_{i \in I} \NN\cap \xxx_i
		\label{lemma_full_Characterization_pt_1}
		.
		\end{equation}
		In particular,~\eqref{lemma_full_Characterization_pt_1} implies that for every open subset $U\subseteq \tilde{\xxx}$
		\begin{equation}
		\emptyset \neq
		\NN\cap \xxx_i  \cap  \left[Q_{\tilde{\xxx}}^{-1}[U]\cap \xxx_i\right]  
		\subseteq 
		U \cap \bigcup_{i \in I} \NN\cap \xxx_i
		\label{lemma_full_Characterization_pt_2}
		.
		\end{equation}
		Therefore, $\bigcup_{i \in I} \NN\cap \xxx_i$ is dense in $\tilde{\xxx}$ and therefore it is dense in $\bigcup_{i \in I} \xxx_i$ equipped with its relative topology.   Hence, $\fff$ has the UAP and therefore (i) holds.    
		
		In the next portion of the proof, we denote the (linear algebraic) dimension of any vector space $V$ by $\dim(V)$.  Recall, that this is the cardinality of the smallest basis for $V$.  We follow the Von Neumann convention and, whenever required by the context, we identify the natural number $n$ with the ordinal $\{1,\dots,n\}$.  
		
		Assume that (i) holds.  For the first part of this proof, we would like to show that $D$ contains a linearly independent and dense subset $D'$.  
		Since $\xxx$ is homeomorphic to some infinite-dimensional Fr\'{e}chet space $E$, then there exists a homeomorphism $\Phi:\xxx\to E$ mapping $\NN$ to a dense subset $D$ of $E$.  We denote the metric on $E$ by $d$.  A consequence of \citep[Theorem 3.1]{phelps1958subreflexive}, discussed thereafter by the authors, implies that since $E$ is an infinite dimensional Fr\'{e}chet space then it has a dense Hamel basis, which we denote by $\{b_a\}_{a \in A}$.  By definition of the Hamel basis of $E$ we may assume that the cardinality of $A$, denoted by $Card(A)$, is equal to $\dim(E)$.  Next, we use $\{b_a\}_{a \in A}$ to produce a base of open sets for the topology of $E$ of cardinality equal to $\dim(E)$.  
		Since $E$ is a metric space, then its topology is generated by the open sets $\{\operatorname{Ball}_{E}(b_a,q)\}_{a \in A, r \in (0,\infty)}$, where
		$
		\operatorname{Ball}_{E}(b_a,r) \triangleq \left\{
		d(b_a,x)<r
		\right\}.
		$  
		Indeed, since $\qq$ is dense in $\rr$, then for every $a \in A$ and $r \in (0,\infty)$ the basic open set $\operatorname{Ball}_{E}(b_a,r)$ can be expressed by
		$
		\operatorname{Ball}_{E}(b_a,r) = \bigcup_{q \in \qq\cap (0,r)} \operatorname{Ball}_{E}(b_a,q).
		$  
		Hence, $\{\operatorname{Ball}_{E}(b_a,q)\}_{a \in A, q \in \qq\cap (0,\infty)}$ generates the topology on $E$.  Moreover, the cardinality the indexing set $A\times \qq$ is computed by 
		$$
		Card(A\times \qq\cap (0,\infty)) = \max\{Card(A),Card(\qq)\}= \max\{\dim(E),Card(\qq)\}=\dim(E),
		$$ since $E$ is infinite and therefore at-least countable.  Therefore, $\{\operatorname{Ball}_{E}(b_a,q)\}_{a \in A, q \in \qq\cap (0,\infty)}$ is a base for the topology on $E$ of Cardinality equal to $\dim(E)$.  
		Let $\omega$ be the smallest ordinal with $Card(\omega)=\dim(E)=Card(A\times \qq \cap (0,\infty))$.  In particular, there exists a bijection $F:\omega \to A\times \qq \cap (0,\infty)$ which allows us to canonically order the open sets $\{\operatorname{Ball}_{E}(F(j)_1,F(j)_2)\}_{j \leq \omega}$, where for any $j<\omega$ we denote $F(j)_1 \in A$ and $F(j)_2 \in \qq \cap (0,\infty)$.  
		
		We construct $D'$ by transfinite induction using $\omega$.  Indeed since $1<\omega$, then since $D$ is dense in $E$ and\\ $\{\operatorname{Ball}_{E}(F(j)_1,F(j)_2)\}_{j \leq \omega}$ defines a base for the topology of $E$, then there exists some\\ $U_1 \in \{\operatorname{Ball}_{E}(F(j)_1,F(j)_2)\}_{j \leq \omega}$ containing some $d_1 \in D$.  For the inductive step, suppose that for all $i\leq j$ for some $j <\omega$, we have constructed a linearly independent set $\{d_i\}_{i< j}$ with $d_i \in \{\operatorname{Ball}_{E}(F(i)_1,F(i)_2)\}$ for every $i\leq j$.  Since $j<\omega$ and $\{d_i\}_{i< j}$ contains $Card(j)$ and $\{d_i\}_{i < j}$ is a Hamel basis of $\operatorname{span}(\{x_i\}_{i < j})$ then 
		$
		\dim\left(
		\operatorname{span}(\{x_i\}_{i < j})
		\right) < \dim(E).
		$ 
		Hence, $\operatorname{span}(\{x_i\}_{i < j})$ has empty interior and therefore it cannot contain any $\{\operatorname{Ball}_{E}(F(j)_1,F(j)_2)\}_{j \leq \omega}$.  In particular, there is an open subset $V'\subseteq \operatorname{Ball}_{E}(F(j)_1,F(j)_2) - \operatorname{span}(\{x_i\}_{i < j})$ and since $D$ was assumed to be dense in $E$ then there must be some $d_j \in V'\subseteq \operatorname{Ball}_{E}(F(j)_1,F(j)_2)$.  This completes the inductive step and therefore there is a linearly independent and dense subset $D'\triangleq \{d_j\}_{j \leq \omega}$ contained in $D$ of cardinality $Card(\omega)=\dim(E)$.  
		
		Next, let $I$ be the set of all countable sequences of distinct elements in $\omega$.  For every $i \in I$, let $E_i\triangleq \overline{\operatorname{span}_{j \in i}(d_j)}$, where $\overline{A}$ denotes the closure of a subset $A\subseteq E$ in the topology of $E$.  Then, each $E_i$ is a linear subspace of $E$ with countable basis $\{d_j\}_{j \in i}$.  Since any Fr\'{e}chet space with countable basis is separable and therefore each $E_i$ is a separable Fr\'{e}chet space.  Moreover, by construction, 
		\begin{equation}
		D'\subseteq \bigcup_{i \in I} E_i \subseteq E
		\label{eq_dense_LHS}
		\end{equation}
		and therefore $\bigcup_{i \in I} E_i$ is dense in $E$ since $D'$ is dense in $E$.  Since $\Phi$ is a homeomorphism then $\Phi^{-1}:E\to \xxx$ is a continuous surjection, and since the image of a dense set under any continuous map is dense in the range of that map then $\Phi^{-1}(D')$ is dense in $\xxx$.  Moreover, using the fact that inverse images commute with unions and the fact that that $\Phi$ is a bijection, we compute that
		\begin{equation}
		\Phi^{-1}(D')\subseteq \Phi^{-1}\left[\bigcup_{i \in I} E_i\right]
		= 
		\bigcup_{i \in I} \Phi^{-1}\left[E_i\right].
		\label{eq_dense_LHS_w}
		\end{equation}
		Since $\Phi$ as a bijection and $D$ was defined as the image of $\NN$ in $E$ under $\Phi$, then $D'\subset \NN$ and $D'$ is dense in $\xxx$.  In particular,~\eqref{eq_dense_LHS_w} implies that $\bigcup_{i \in I} \Phi^{-1}[E_i] \subseteq \bigcup_{i \in I} (\NN \cap \Phi^{-1}[E_i])$ and therefore $\bigcup_{i \in I} (\NN \cap \Phi^{-1}[E_i])$ is dense in $\xxx$.  
		In particular, $\bigcup_{i \in I} \Phi^{-1}[E_i]$ is dense in $\xxx$, and for each $i \in I$, if we define $\xxx_i\triangleq \Phi^{-1}[E_i]$ then we obtain (ii.a).  
		Since $\Phi$ is a homeomorphism then it preserves dense sets and in particular since $\{d_i\}_{j \in i}$ is a countable, dense, and linearly independent subset of $\Phi^{-1}[\{d_j\}_{j \in i}]$ then it is a dense countable subset of $\xxx_i$.  Hence, each $\xxx_i$ is separable.  
		This gives (ii.b).  Lastly, by \cite{kadets1967proof} any two separable infinite-dimensional Fr\'{e}chet space are homeomorphic.  In particular, since $L^2(\rr)$ is a separable Hilbert space is a separable Fr\'{e}chet space.  Therefore, for each $i \in I$, there is a homeomorphism $\Phi_i: E_i \to L^2(\rr)$.  In particular, $\Phi_i\circ \Phi:\xxx_i\to L^2(\rr)$ must be a homeomorphism and therefore (ii.b) holds.  Therefore, (i) implies (ii).  
		
		Suppose that (ii) holds.  Then, (iii.a) holds by (ii.a).  For each $i \in I$, let $\{d_{n,i}\}_{n \in \nn}$ be a countable dense subset of $\xxx_i\cap \NN$ for which $\Phi(\{d_{n,i}\}_{n \in \nn})$ is a linearly independent, and let $E_i=\overline{\operatorname{span}(\{d_{n,i}\}_{n \in \nn})}$.  Let $D\triangleq \bigcup_{i \in I} \{d_{n,i}\}_{n \in \nn}$ and $D'\triangleq \Phi(D)$.  Thus, for every $i \in I$, $D'\cap E_i$ is a countably infinite linearly independent and dense subset of $E_i$ then by \citep[Theorem 8.24]{LinearChaossogooodos} there exists a continuous linear operator $T_i:D\cap E_i \to D\cap E_i$ satisfying
		$$
		T_i^n(d_{n,i})=d_{n+1,i},
		$$
		for each $n \in \nn$ and each $i \in I$.  In particular, $\left\{T^n_i(d_{0,i})\right\}$ is dense in $E_i$.  For each $i \in I$, define $\phi_i\triangleq  \Phi^{-1}\circ T_i \circ \Phi$ and $g_i\triangleq \Phi^{-1}(d_{0,i})$ and observe that for every $n \in \nn$
		\begin{equation}
		\begin{aligned}
		\phi^n_i(g_i) = & \underbrace{(\Phi^{-1}\circ T_i\circ \Phi)\circ \dots \circ (\Phi^{-1}\circ T_i\circ \Phi)}_{n-times}
		(\Phi^{-1}(d_{i,0}))
		\\
		= & \Phi^{-1}\circ T_i^n(d_{0,i}) 
		.
		\end{aligned}
		\label{eq_characterization_convenient}
		\end{equation}
		Since $\{T_i^n(d_{0,i})\}_{n \in \nn}$ is dense in $E_i$ and $\Phi$ is a homeomorphism from $\xxx_i$ to $E_i$ then 
		$$
		\Phi^{-1}\left[
		\{T_i^n(d_{0,i})\}_{n \in \nn}
		\right]= \left\{
		\phi_i^n(g_i)
		\right\}_{n \in \nn}
		$$ 
		is dense in $\xxx_i$.  Thus, (iii.c) holds.  For any $i \in I$, define the map $\psi_i:L^2(\rr)\to L^2(\rr)$ by
		$$
		\psi_i \triangleq  (\Phi_i\circ \Phi)^{-1} \circ \phi_i \circ (\Phi_i\circ \Phi),
		$$
		and define the vector $\tilde{g}_i \in L^2(\rr)$ by $\tilde{g}_i\triangleq \Phi_i\circ \Phi(g_i)$.  Since $\Phi$ and $\Phi_i$ are homeomorphisms and since $\phi_i$ is continuous then $\psi_i$ is well-defined and continuous.  Moreover, analogously to~\eqref{eq_characterization_convenient} we compute that 
		$
		\left\{
		\psi_i^n(\tilde{g}_i)
		\right\}_{n \in \nn}
		$ is dense in $L^2(\rr)$.  Since $L^2(\rr)$ is a complete separable metric space with no isolated points and $\psi_i$ is continuous self-map of $L^2(\rr)$ for which there is a vector $\tilde{g}_i \in L^2(\rr)$ such that the set of iterates $\{\psi_i^n(\tilde{g}_i)\}_{n \in \nn}$ is dense in $L^2(\rr)$ then Birkhoff Transitivity Theorem, see the formulation of \cite[Theorem 1.16]{LinearChaossogooodos}, implies that for every pair of non-empty open subsets $\tilde{U},\tilde{V}\subseteq L^2(\rr)$ there is some $n_{\tilde{U},\tilde{V}}$ satisfying
		\begin{equation}
		\phi^{n_{\tilde{U},\tilde{V}}}(\tilde{U})\cap \tilde{V} \neq \emptyset
		\label{eq_Birkoff_almost_done}
		.
		\end{equation}
		Since $\Phi_i\circ \Phi$ is a homeomorphism, then \citep[Proposition 1.13]{LinearChaossogooodos} and~\eqref{eq_Birkoff_almost_done} imply that for every pair of non-empty open subsets $U',V'\subseteq \xxx_i$ there exists some $n_{U',V'} \in \nn$ satisfying
		\begin{equation}
		\phi^{n_{U',V'}}(U')\cap V' \neq \emptyset
		\label{eq_Birkoff_almost_done_1}
		.
		\end{equation}
		Since $\xxx_i$ is equipped with the subspace topology then every non-empty open subset $U'\subseteq \xxx_i$ is of the form $U\cap \xxx_i$ for some non-empty open subset $U\subseteq \xxx$.  Therefore, ~\eqref{eq_Birkoff_almost_done_1} implies (iii.b).  Since both $L^2(\rr)$ and $C(\rr)$ are separable infinite-dimensional Fr\'{e}chet spaces then the \citep[Anderson-Kadec Theorem]{kadets1967proof} implies that there exists a homeomorphism $\Psi:L^2(\rr)\rightarrow C(\rr)$.  Therefore, for each $i \in I$, $\Psi\circ \Phi_i\circ \Phi:\xxx\rightarrow C(\rr)$ is a homeomorphism and thus (ii.c) implies (iii.d).

		Suppose that (iii) holds.  For every $i \in I$, set $X_i\triangleq \xxx_i$,  let $\Phi_i\triangleq 1_{X_i}$ be the identity map on $X_i$, set $\psi_i\triangleq \phi_i$, and set $x_i\triangleq g_i$.  Therefore, (iv) holds.  
		
		Suppose that (iv) holds.  By (iv.c), for each $i \in I$, $\NN\cap \xxx_i$ is dense in $\xxx_i$.  Therefore,
		\begin{equation}
		\bigcup_{i \in I} \xxx_i 
		=
		\bigcup_{i \in I} \overline{\NN \cap \xxx_i}
		\subseteq 
		\overline{\bigcup_{i \in I} \NN \cap \xxx_i}	
		\subseteq \xxx
		\label{eq_first_inclusion}
		.
		\end{equation}
		By (iv.a) since $\bigcup_{i \in I} \xxx_i$ is dense in $\xxx$ therefore its closure is $\xxx$ and therefore the smallest, and thus only, closed set containing $\bigcup_{i \in I}\xxx_i$ is $\xxx$ itself.  Therefore, by~\eqref{eq_first_inclusion} the smallest set containing $\bigcup_{i \in I} \NN \cap \xxx_i$ must be $\xxx$.  Therefore, $\NN$ is dense in $\xxx$ and (i) holds.  This concludes the proof.  
	\end{proof}
	\begin{proof}[{Proof of Theorem~\ref{thrm_Construction_Theorem}}]
		By the \citep[Anderson-Kadec Theorem]{kadets1967proof} there is no loss of generality in assuming that $m=n=1$, since $C(\rrm,\rrn)$ and $C(\rr)$ are homeomorphic.  Let $\xxx'\triangleq \bigcup_{i \in I} \Phi_i(C(\rr))$.  By~\eqref{eq_density_condition}, $\xxx'$ is dense in $\xxx$ and since density is transitive, then it is enough to show that $\bigcup_{i \in I} \Phi_i(\NN)$ is dense in $\xxx'$ to conclude that it is dense in $\xxx$.  Since each $\Phi_i$ is continuous, then, the topology on $\xxx'$ is no finer than the finest topology on $\bigcup_{i \in I} \Phi_i(C(\rr))$ making each $\Phi_i$ continuous and by \citep[Proposition 2.6]{BourbakiTopGen} such a topology exists.  Let $\xxx''$ denote $\bigcup_{i \in I} \Phi_i(C(\rr))$ equipped with the finest topology making each $\Phi_i(C(\rr))$ into a subspace.  By construction, if $U\subseteq \xxx'$ is open then it is open in $\xxx''$ and therefore if $\bigcup_{i \in I} \Phi_i(\NN) $ intersects each non-empty open subset of $\xxx''$ then it must do the same for $\xxx'$.  Hence, it is enough to show that $\bigcup_{i \in I} \Phi_i(\NN)$ is dense in $\xxx''$ to conclude that it is dense in $\xxx'$ and therefore, $\bigcup_{i \in I} \Phi_i(\NN)$ is dense in $\xxx$.  
		
		We proceed similarly to the proof of Lemma~\ref{lemma_full_Characterization}.  
		Indeed, by \citep[Proposition 2.7]{BourbakiTopGen} the space $\xxx''$ is given by the (topological) quotient of the disjoint union $\sqcup_{i \in I} \Phi_i(C(\rr))$, in the sense of topological spaces (see \citep[Example 3, Section 2.4]{BourbakiTopGen}), under the equivalence relation $f_i\sim f_j$ if $f_i=f_j$ in $\xxx$.  Denote the corresponding quotient map by $Q_{\xxx'}$.  Since a subset $U$ of the quotient topology is open (see \citep[Example 2, Section 2.4]{BourbakiTopGen}) if and only if $Q_{\xxx'}^{-1}[U]$ is an open subset of $\sqcup_{i \in I} \Phi_i(C(\rr))$ and since a subset $V$ of $\sqcup_{i \in I} \Phi_i(C(\rr))$ is open if and only if $V\cap \Phi_i(C(\rr))$ is open for each $i \in I$ in the topology of $\Phi_i(C(\rr))$ then $U\subseteq \xxx''$ is open if and only if $Q_{\xxx'}^{-1}[U] \cap \Phi_i(C(\rr))$ is open for each $i \in I$.  Since $\{\NN\cap \Phi_i(C(\rr))\}_{n \in \nn^+}$ is dense in $\Phi_i(C(\rr))$ then for every open subset $U'\subseteq \Phi_i(C(\rr))$
		\begin{equation}
		\emptyset \neq U' \cap 
		\NN\cap \Phi_i(C(\rr)) 
		\subseteq
		U' \cap \bigcup_{i \in I} \NN\cap \Phi_i(C(\rr))
		\label{lemma_full_Characterization_pt_1a}
		.
		\end{equation}
		In particular,~\eqref{lemma_full_Characterization_pt_1a} implies that for every open subset $U\subseteq \xxx''$
		\begin{equation}
		\emptyset \neq
		\NN\cap \Phi_i(C(\rr))  \cap  \left[Q_{\xxx'}^{-1}[U]\cap \Phi_i(C(\rr))\right]  
		\subseteq 
		U \cap \bigcup_{i \in I} \NN\cap \Phi_i(C(\rr))
		\label{lemma_full_Characterization_pt_2b}
		.
		\end{equation}
		Therefore, $\bigcup_{i \in I} \NN\cap \Phi_i(C(\rr))$ is dense in $\xxx''$ and therefore it is dense in $\bigcup_{i \in I} \Phi_i(C(\rr))$ equipped with its relative topology.   Hence, $\arch[\Phi]$ has the UAP on $\xxx''$ and therefore it has the UAP on $\xxx$ itself.  
	\end{proof}
	\begin{proof}[{Proof of Theorem~\ref{thrm_Meta_Universal}}]
		Let $\sigma$ be a continuous and non-polynomial activation function.  Then \cite{pinkus1999approximation} implies that the architecture $\arch[0]$, as defined in Example~\ref{ex_fully_connected_DffNNS}, is a universal approximator on $C(\rr)$.  
		
		By Theorem~\ref{thrm_Characterization_dynamics}, since $\arch$ has the UAP on $\xxx$ and since $\xxx$ is homeomorphic to an infinite-dimensional Fr\'{e}chet space then there are homeomorphisms $\{\Phi_i\}_{i \in I}$ from $C(\rr)$ onto a family of subspaces $\{\xxx_i\}_{i \in I}$ of $\xxx$ such that $\bigcup_{i \in I} \xxx_i$ is dense.  Fix $\epsilon>0$ and $f \in \xxx$.  
		Since $\bigcup_{i \in I} \xxx_i$ is dense in $\xxx$ there exists some $i \in I$ and some $f_i\in \xxx_i$ such that
		\begin{equation}
		d_{\xxx}(f,f_i)<\frac{\epsilon}{2}
		\label{proof_dont_looktofar_thrm_Meta_Universal}
		.
		\end{equation}
		Since $\Phi_i$ is a homeomorphism then it must map dense sets to dense sets.  Since $\arch[0]$ has the UAP on $C(\rr)$ then $\NN[\arch[0]]$ is dense in $C(\rr)$ and therefore, for each $i \in I$, $\Phi_i(\NN[\arch[0]])$ is dense in $\xxx_i$.  Hence, there exists some $\tilde{g}_{\epsilon}\in \Phi_i(\NN[\arch[0]])$ such that $d_{\xxx}(f_i,\tilde{g}_{\epsilon})<\frac{\epsilon}{2}$.  Since $\Phi_i$ is a homeomorphism, it is a bijection, therefore there exists a unique $g_{\epsilon}\in \NN[\arch[0]]$ with $\Phi_i(g_{\epsilon})=\tilde{g}_{\epsilon}$.  Hence, the triangle inequality and~\eqref{proof_dont_looktofar_thrm_Meta_Universal} imply that
		\begin{equation}
		d_{\xxx}\left(
		f,\Phi_i(g_{\epsilon})
		\right) \leq 
		d_{\xxx}\left(
		f,f_i
		\right)
		+
		d_{\xxx}\left(
		f_i,\Phi_i(g_{\epsilon})
		\right) < \epsilon
		\label{proof_dont_looktofar_thrm_Meta_Universal_2}
		.
		\end{equation}
		This yields the first inequality in the Theorem's statement.  
		
		By Theorem~\ref{thrm_Characterization_dynamics} since, for each $i \in I$, $\NN\cap \xxx_i$ is dense in $\xxx_i$ and since $\Phi_i^{-1}$ is a homeomorphism on $\xxx_i$ then $\Phi_i^{-1}\left(\NN\cap \xxx_i\right)$ is dense in $C(\rr)$.  In particular, there exits some $\tilde{f}_{\epsilon} \in \Phi_i^{-1}\left(\NN\cap \xxx_i\right)$ satisfying
		\begin{equation}
		d_{ucc}\left(
		g_{\epsilon}(x)
		,
		\tilde{f}_{\epsilon}(x)
		\right)
		<\epsilon
		\label{proof_dont_looktofar_thrm_Meta_Universal_3}
		.
		\end{equation}
		Since $\Phi_i$ is a bijection then there exists a unique $f_{\epsilon}\in \NN$ such that $\Phi_i^{-1}(f_{\epsilon})=\tilde{f}_{\epsilon}$.  Therefore,~\eqref{proof_dont_looktofar_thrm_Meta_Universal_3} and the triangle inequality imply that
		$$
		d_{ucc}\left(
		g_{\epsilon}(x)
		,
		\Phi_i^{-1}(f_{\epsilon})(x)
		\right)
		<\epsilon
		.
		$$
		Therefore the conclusion holds.  
	\end{proof}
	\begin{rremark}\label{remark_meta_universalis}
		By the \citep[Anderson-Kadec Theorem]{kadets1967proof}, since both $L^2(\rr)$ and $C(\rr)$ are separable infinite-dimensional Fr\'{e}chet spaces then there exists a homeomorphism $\Phi:L^2(\rr)\rightarrow C(\rr)$.  Therefore, the proof of Corollary~\ref{thrm_Meta_Universal} holds (mutatis mutandis) with each $\Phi$ replaced by $\Phi_i\circ \Phi^{-1}$ and with $C(\rr)$ in place of $L^2(\rr)$.  
	\end{rremark}
	The proof of the next result relies on some aspects of \textit{inductive limits} of Banach spaces.  Briefly, an inductive limit of Banach spaces is a locally convex space $B$ for which there exists a pre-ordered set $I$, a set of Banach sub-spaces $\{B_i\}_{i \in I}$ with $B_i\subseteq B_j$ if $i\leq j$.  The inductive limit of this direct system is the subset $\bigcup_{i \in I} B_i$ equipped with the finest topology which simultaneously makes each $B_i$ into a subspace and makes $\bigcup_{i \in I} B_i$ into a locally-convex spaces.  Spaces constructed in this way are called \textit{ultrabornological spaces} and more details about them can be found in \citep[Chapter 6]{PerezBonetBarrelLCSs}.  
	\begin{proof}[{Proof of Theorem~\ref{thrm_Existence}}]
		Since $B(\xxx_0)$ and $B(X)$ are both infinite-dimensional Banach spaces, then they are infinite-dimensional ultrabornological space, in the sense of \citep[Definition 6.1.1]{PerezBonetBarrelLCSs}.  Since $X$ is separable, then as observed in \cite{GodefroyLipfReeBan}, $B(X)$ is separable.  Therefore, \citep[Theorem 6.5.8]{PerezBonetBarrelLCSs} applies; hence, there exists a directed set $I$ with pre-order $\leq$, a collection of Banach subspaces $\{B_i\}_{i \in I}$ satisfying (i) and (ii), and a collection of continuous linear isomorphisms $\Phi_i:B(X)\rightarrow B_i$.  Furthermore, the topology on $B$ is coarser than the inductive limit topology $\varinjlim_{i \in I} B_i$.  Since each $B(X)$ and $B_i$ are Banach spaces, and in particular normed linear spaces, then by the results of \citep[Section 2.7]{KreyszigIntroFunctionalandApplications1989} the maps $\Phi_i$ are bounded linear isomorphisms.  
		
		Let $i \in I$, and fix any $x_i \in X-\{0_X\}$ then since $\delta^X:X\rightarrow B(X)$ is base-point preserving then $\delta^X_{x_i}\neq 0$ and therefore there exists a linearly independent subset $\mathcal{B}_{x_i}$ of $B(X)$ containing $\delta^X_{x_i}$.  Since $B(X)$ is separable then $\mathcal{B}_{x_i}$ is countably infinite and therefore \citep[Theorem 8.24]{LinearChaossogooodos} there exists a bounded linear map $\phi_i:B(X)\rightarrow B(X)$ such that $\{\phi_i^n(\delta^X_{x_i})\}_{n \in \nn^+}$ is a dense subset of $B(X)$.  
		
		Since $\Phi_i$ is a continuous linear isomorphisms then it is in particular a surjective continuous map from $B(X)$ onto $B_i$.  Since the image of a dense set under a continuous surjection is itself dense then $\left\{\Phi_i\circ \phi_i^n(\delta_{x_i})\right\}_{n \in \nn^+}$ is a dense subset of $B_i$.  Moreover, this holds for each $i \in I$.  
		
		By definition, the topology on $\varinjlim_{i \in I} B_i$ is at-least as fine as the Banach space topology on $B(\xxx_0)$, since each $B_i$ is a linear subspace of $B(\xxx_0)$.  Moreover, the topology on $\varinjlim_{i \in I} B_i$ is no finer than the finest topology on $\bigcup_{i \in I} B_i$ making each $B_i$ into a topological space (but not requiring that $\bigcup_{i \in I} B_i$ be locally-convex), which exists by \citep[Proposition 6]{BourbakiTopo}.  Denote this latter space by $\tilde{B}$.  Therefore, if 
		\begin{equation}
		\bigcup_{i \in I;\, n \in \nn^+} \left\{\Phi_i\circ \phi_i^n(\delta_{x_i})\right\}
		\label{thrm_Existence_proof_target_to_make_dense}
		,
		\end{equation}
		is dense in $\tilde{B}$ then it is dense in $\varinjlim_{i \in I} B_i$ and in $B(\xxx_0)$.  Hence, we show that~\eqref{thrm_Existence_proof_target_to_make_dense} is dense in $\tilde{B}$.  That is, it is enough to show that every open subset of $\tilde{B}$ contains an element of~\eqref{thrm_Existence_proof_target_to_make_dense}.  
		
		By \citep[Proposition 2.7]{BourbakiTopGen} the space $\tilde{B}$ is given by the topological quotient of the disjoint union $\sqcup_{i \in I} B_i$, in the sense of topological spaces (see \citep[Example 3, Section 2.4]{BourbakiTopGen}), under the equivalence relation $x_i\sim x_j$ for any $i\leq j$ if $x_i=x_j$ in $B_j$.  Denote the corresponding quotient map by $Q_{\tilde{B}}$.  Since a subset $U$ of the quotient topology is open (see \citep[Example 2, Section 2.4]{BourbakiTopGen}) if and only if $Q_{\tilde{B}}^{-1}[U]$ is an open subset of $\sqcup_{i \in I} B_i$ and since a subset $V$ of $\sqcup_{i \in I} B_i$ is open if and only if $V\cap B_i$ is open for each $i \in I$ in the topology of $B_i$ then $U\subseteq \tilde{B}$ is open if and only if $Q_{\tilde{B}}^{-1}[U] \cap B_i$ is open for each $i \in I$.  Since $\{\Phi_i\circ \phi_i^n(x_i)\}_{n \in \nn^+}$ is dense in $B_i$ then for every open subset $U'\subseteq B_i$
		\begin{equation}
		\emptyset \neq U' \cap \{\Phi_i\circ \phi_i^n(x_i)\}_{n \in \nn^+} \subseteq U' \cap \bigcup_{i \in I;\, n \in \nn^+} \left\{\Phi_i\circ \phi_i^n(\delta_{x_i})\right\}
		\label{thrm_Existence_proof_target_to_make_dense_non_trivial_intersection}
		.
		\end{equation}
		In particular,~\eqref{thrm_Existence_proof_target_to_make_dense_non_trivial_intersection} implies that for every open subset $U\subseteq \tilde{B}$
		\begin{equation}
		\emptyset \neq
		\{\Phi_i\circ \phi_i^n(x_i)\}_{n \in \nn^+} \cap  \left[Q_{\tilde{B}}^{-1}[U]\cap B_i\right]  
		\subseteq 
		\bigcup_{i \in I;\, n \in \nn^+} \left\{\Phi_i\circ \phi_i^n(\delta_{x_i})\right\} \cap U 
		\label{thrm_Existence_proof_target_to_make_dense_non_trivial_intersection2}
		.
		\end{equation}
		Therefore,~\eqref{thrm_Existence_proof_target_to_make_dense} is dense in $\tilde{B}$ and, in particular, it is dense in $B(\xxx_0)$.  
		
		Since $\xxx_0$ was barycentric, then there exists a continuous linear map $\rho:B(\xxx_0)\rightarrow \xxx_0$ which is a left-inverse of $\delta^{\xxx_0}$.  Thus, for every $f \in \xxx_0$, $\rho\circ \delta^{\xxx_0}_f = f$ and therefore $\rho$ is a continuous surjection.  Since the image of a dense set under a continuous surjection is dense and since~\eqref{thrm_Existence_proof_target_to_make_dense} is dense then
		\begin{equation}
		\bigcup_{i \in I;\, n \in \nn^+} \left\{\rho\circ \Phi_i\circ \phi_i^n(\delta_{x_i})\right\}
		\label{thrm_Existence_proof_target_to_make_dense_3}
		,
		\end{equation}
		is a dense subset of $\xxx_0$.  Since $\xxx_0$ has assumed to be dense in $\xxx$ and since density is transitive then~\eqref{thrm_Existence_proof_target_to_make_dense_3} is dense in $\xxx$.  This concludes the main portion of the proof.   
		
		The final remark follows from the fact that if $X=\xxx_0$ then the identity map $1_{X}:X\rightarrow \xxx_0$ is an isometry and therefore the universal property of $B(X)$ described in Theorem~\citep[Theorem 3.6]{WeaverNice} implies that $1_X$ uniquely extends to a bounded linear isomorphism $L$ between $B(X)$ and $B(\xxx_0)$ satisfying
		$$
		L\circ \delta^X = \delta^{\xxx_0}\circ 1_{X} = \delta^{\xxx_0} \mbox{ and } 
		L^{-1}\circ \delta^{\xxx_0} = \delta^{X}\circ 1_X^{-1} = \delta^X
		.
		$$
		Hence $L$ must be the identity on $B(X)$.  
	\end{proof}
	\section{Proof of Applications of Main Results}\label{Appendix_B_Proof_of_applications_to_main_results}
	\begin{lem}\label{lem_activation_gaps_Birkohoff}
		Fix some $b \in \rrm$, and let $\sigma:\rr\to\rr$ be a continuous activation function.  Then $\Phi_{A,b}$ is a well-defined and continuous linear map from $C(\rrm,\rrn)$ to itself and the following are equivalent:
		\begin{enumerate}[(i)]
			\item For each $\delta>0,\epsilon>0$ and each $f,g\in C(\rrm,\rrn)$ there is some $N_{U,V}\in \nn^+$ such that
			$$
			\left\{\Phi^{N_{U,V}}(\tilde{g}): \, d_{ucc}(\tilde{g},g)<\delta\right\} \cap \left\{
			\tilde{f}: \, d_{ucc}(\tilde{f},f)<\epsilon
			\right\} \neq \emptyset
			,
			$$
			\item $\sigma$ is injective, $A$ is of full-rank, and for every compact subset $K\subseteq [a,b]$ there is some $N_K\in \nn^+$ such that 
			$$
			S^N(K)\cap K = \emptyset,
			$$
			where $S(x)=\sigma\bullet (Ax+b)$.  
		\end{enumerate}
		If $A$ is the $m\times m$-identity matrix $I_m$ and $b_i>0$ for $i=1,\dots,m$ then (i) and (ii) are equivalent to
		\begin{enumerate}[(iii)]
			\item $\sigma$ is injective and has no fixed-points.   
		\end{enumerate}
		If $A$ is the $m\times m$-identity matrix $I_m$ and $b_i>0$ for $i=1,\dots,m$ then (iii) is equivalent to
		\begin{enumerate}[(iv)]
			\item Either $\sigma(x)>x$ or $\sigma(x)<x$ for every $x \in \rr$.  
		\end{enumerate}
	\end{lem}
	\begin{proof}[{Lemma~\ref{lem_activation_gaps_Birkohoff}}]
		By \citep[Theorem 46.8]{munkres2014topology} the topology of uniform convergence on compacts is the compact-open topology on $C(\rrm,\rrn)$ and by \citep[Theorem 46.11]{munkres2014topology} composition is a continuous operation in the compact-open topology.  Therefore, $\Phi_{A,b}$ is well-defined and continuous map.  Its linearity follows from the fact that
		$$
		\Phi_{A,b}(af+g) = (af_g)\circ S = a(f\circ S) + g\circ S.
		$$
		
		Since the topology of uniform convergence on compacts is a metric topology, with metric $d_{ucc}$, then \\$\left\{U_{f,\epsilon}:f \in C(\rrm,\rrn),\, \epsilon>0\right\}$ defines a base for this topology, where $U_{f,\epsilon}\triangleq \left\{g \in C(\rrm,\rrn):\, d_{ucc}(f,g)<\epsilon\right\}$.  Therefore, Lemma~\ref{lem_activation_gaps_Birkohoff} (i) is equivalent to the statement: for each pair of non-empty open subsets $U,V \in C(\rrm,\rrn)$ there is some $N_{U,V}\in \nn^+$ such that
		$
		\Phi_{I,b}^{N_{U,V}}(U)\cap V \neq \emptyset
		.
		$  
		Without loss of generality, we prove this formulation instead.  
		
		Next, by \citep[Corollary 4.1]{KalmesDynamicsWeightCompOpLocFunctSpaces2019} $\Phi_{A,b}$ satisfies Theorem~\ref{thrm_Characterization_dynamics} (ii.b) if and only if $S(x)\triangleq  \sigma(Ax+b)$ is injective and for every compact subset $K\subseteq \rrm$ there exists some $N_K \in \nn^+$ such that 
		\begin{equation}
		S^{N_K}(K)\cap K = \emptyset
		\label{eq_runaway}
		.
		\end{equation}  
		Therefore, $A$ must be injective which is only possible if $A$ is of full-rank.  This gives the equivalence between (i) and (ii).  
		
		We consider the equivalence between (ii) and (iii) in the case where $A$ is the identity matrix and $b_i>0$ for $i=1,\dots,m$.  Since $S(x)=(\sigma(x+b_1),\dots,\sigma(x+b_m))$ it is sufficient to verify condition~\eqref{eq_runaway} in the case where $m=1$.  
		Since $b_i>0$ for $1,\dots,m$ then it is clear that $S$ is injective and has no fixed points if and only if $\sigma$ is injective and has no fixed points.  We show that $S$ is injective and has no fixed points if and only if (ii) holds.  
		Indeed, note that if $S$ has not fixed points, then since $b_i>0$ for $i=1,\dots,m$ then $S$ has no fixed points if and only if $\sigma$ no fixed points.  
		
		From here, we proceed analogously to the proof of \citep[Lemma 4.1]{PrzestackiHypercycliconContinuous}.  
		If $S$ has a fixed-point then for every $N \in \nn^+$, $S^N({x})=\{x\}$ which is a non-empty compact subset of $\rr$.  Therefore,~\eqref{eq_runaway} cannot hold.  Conversely, suppose that $S$ has no fixed points.  The intermediate-value theorem and the fact that $S$ has no fixed-points that either $S(x)<x$ or $S(x)>x$.  Mutatis mutandis, we proceed with the first case.  Since $\sigma$ is injective and $S$ has not fixed points then $S$ must be a strictly increasing function; thus $S([a,b])=[S(a),S(b)]$ for every $a<b$.  
		
		Let $K$ be a non-empty compact subset of $\rr$.  By the Heine-Borel theorem $K$ is closed and bounded, thus it is contained in some $[a,b]$ for $a<b$.  Therefore, it is sufficient to show the results for the case where $K=[a,b]$.  Since $S$ is increasing then for every $n \in \nn$, the sequence $\{S^n(a)\}_{n \in \nn}$ satisfies $S^n(a)<S^{n+1}(a)$.  If this sequence is not unbounded then there would exist some $a_0 \in \rr$ such that $a_0= \lim\limits_{n \to \infty} S^n(a)$.  Therefore, by the continuity of $S$ we would find that
		$$
		a_0 
		= 
		\lim\limits_{n \to \infty} S^n(a)
		= 
		\lim\limits_{n \to \infty} S^{n+1}(a)
		=
		\lim\limits_{n \to \infty} S(S^{n}(a))
		=
		S\left(
		\lim\limits_{n \to \infty} S^n(a)
		\right)
		=
		S(a_0),
		$$
		but since $S$ has not fixed points then there cannot exist such an $a_0$ since otherwise $a_0=S(a_0)$.  Therefore, $a_0$ does not exist and thus $\{S^n(a)\}_{n \in \nn}$ is unbounded.  Hence, for every $a<b$ there exists some $N_{[a,b]}\in \nn^+$ such that
		$$
		S^{N_{[a,b]}}([a,b])\cap [a,b] = \emptyset.
		$$
		Thus, (ii) and (iii) are equivalent when $A=I_m$.  
		
		Next, assume that any of (i) to (iii) hold, that $\xxx$ is a non-empty subset of $C(\rrm,\rrn)$, and that $\arch$ has the UAP on $\xxx$.  Then for any other non-empty open subset $U\subseteq C(\rrm,\rrn)$ there exists some $N_{\xxx,U}\in \nn$ such that 
		\begin{equation}
		\Phi_{A,b}^{N_{\xxx,U}}[\xxx] \cap U \neq \emptyset 
		\label{cor_UAP_extension_proof_eq_intersector}
		.
		\end{equation}
		Since $\Phi_{A,b}$ is continuous then so is $\Phi_{A,b}^N$ and therefore $(\Phi_{A,b}^{N_{\xxx,U}})^{-1}[U]$ is a non-empty open subset of $C(\rrm,\rrn)$.  Since the finite intersection of open sets is again open, then we have that
		\begin{equation}
		(\Phi_{A,b}^{N_{\xxx,U}})^{-1}\left[
		\Phi_{A,b}^{N_{\xxx,U}}[\xxx] \cap U
		\right]
		= \xxx \cap \Phi_{A,b}^{N_{\xxx,U}}[U]
		\label{cor_UAP_extension_proof_eq_intersector_depth_howhow}
		.
		\end{equation}
		This implies that $\xxx \cap \Phi_{I_m,b}^{N_{\xxx,U}}[U]$ is a non-empty open subset of $C(\rrm,\rrn)$ contained in $\xxx$.  Since $\arch$ has te UAP on $\xxx$, then there exists some $f \in \NN \cap [\xxx \cap \Phi_{A,b}^{N_{\xxx,U}}[U]]$.  Thus, $\Phi^{N_{\xxx,U}}(f)\in U$ and, by definition, $\Phi^{N_{\xxx,U}}(f)\in \NN[\arch[\sigma;deep]]$.  
		
		Thus, for each $U$ in
		\begin{equation}
		\left\{
		\left\{
		g \in C(\rrm,\rrn) d_{ucc}(g ,f)<\epsilon
		\right\}
		\right\}_{f \in C(\rrm,\rrn), \epsilon>0}
		\label{eq_collection}
		,
		\end{equation}
		there exists some $N_U \in \nn^+$ and some $f_U \in \NN$ such that $\Phi^{N_U}(f_U)\in U$.  In particular, since~\eqref{eq_collection} is a base for the topology on $C(\rrm,\rrn)$ and since the intersection of open sets is again open, then every non-empty open subset of $U$ is contained an element of~\eqref{eq_collection} which, in turn, contains an element of the form $\Phi^{N_U}(f_U)$.  Thus, $\NN[\arch[\sigma;deep]] \cap U\neq \emptyset$.  %
		Hence, $\NN[\arch[\sigma;deep]]$ has the UAP on $C(\rrm,\rrn)$.  
	\end{proof}
	\begin{proof}[{Proof of Theorem~\ref{thrm_activation_gaps}}]
		The equivalence between (i), (ii), and (iv) follows from Lemma~\ref{lem_activation_gaps_Birkohoff}.  The equivalence between (iii) and (iv) follows from the formulation of Birkhoff's transitivity theorem described in \citep[Theorem 2.19]{LinearChaossogooodos}.  
	\end{proof}
	\begin{proof}[Proof of Proposition~\ref{prop_examples_of_good_activations}]
		Since $\alpha_1<1$ then $\sigma(x)>x$ for every $x<0$.  Since $0< \alpha_2$ then $\sigma(0)=0< \alpha_2$.  Lastly, since $\tilde{\sigma}$ is monotone increasing then for every $x>0$ we have that
		$$
		\sigma(x) > x + \alpha_2 >x.
		$$
		Therefore, $\sigma$ cannot have a fixed point.  Moreover, since $\tilde{\sigma}$ is strictly increasing it must be injective, since if $x<y$ then $\sigma(x)<\sigma(y)$ and therefore $\sigma(x)\neq \sigma(y)$ if $x\neq y$.  Hence, $\sigma$ is injective.  Moreover, since the sum of continuous functions is again continuous, then $\sigma$ is continuous.  
		
		Since $\alpha_1x + \alpha_2$ is affine then it is continuously differentiable.  Thus $\sigma$ is continuously differentiable on any $x<0$.  Lastly, setting $\alpha_2$ not equal to $\tilde{\sigma}'(0)-1$ ensure that $\sigma$ is not differentiable at $0$ and therefore it cannot be polynomial.  In particular, it cannot be affine.  
	\end{proof}
	For convenience, we denote the collection of set-functions from $\rrm$ to $\rrn$ by $[\rrm,\rrn]$.  
	\begin{proof}[{Proof of Corollary~\ref{cor_biased_nets}}]
		Since $d_{ucc}$ is a metric on $[\rrm,\rrn]$ and since $C(\rrm,\rrn)\subseteq [\rrm,\rrn]$, then the map $F:C(\rrm,\rrn)\rightarrow C(\rrm,\rrn)$ defined by $F(g)\triangleq d_{ucc}(\tilde{f}_0,g)$ is continuous.  Therefore, the set $F^{-1}\left[(-\infty,\delta)\right]$ is an open subset of $C(\rrm,\rrn)$.  In particular,~\eqref{eq_regularitycondition} guarantees that it is non-empty.  
		Since $\sigma$ is non-affine and continuously differentiable at-least at one point with non-zero derivative at that point then \citep[Theorem 3.2]{kidger2019universal} applies, whence the set $\xxx_0$ of continuous functions $h:\rrm\rightarrow \rrn$ with representation
		$$
		h(x)= W_J\circ \sigma \bullet \dots \circ \sigma \bullet W_1,
		$$
		where $W_j:\rrflex{d_j}\rightarrow \rrflex{d_{j+1}}$, for $j=1,\dots,J-1$, are affine and $n_m+2\geq d_j$ if $j\not\in\{1,J\}$ and $d_{1}=m$, and $d_J=n$, is dense in $C(\rrm,\rrn)$.   Therefore, since $F^{-1}\left[(-\infty,\delta)\right]$ is an open subset of $C(\rrm,\rrn)$ then $\xxx_0\cap F^{-1}\left[(-\infty,\delta)\right]$ is dense in $F^{-1}\left[(-\infty,\delta)\right]$.  
		
		Fix some $b \in \rrm$ with $b_i>0$ for $i=1,\dots,m$.  Since $\sigma$ is continuous, injective, and has no fixed-points then applying Lemma~\ref{lem_activation_gaps_Birkohoff} implies that 
		$
		\xxx_1 \triangleq \{\Phi_{I_m,b}^N(f):\, f \in F^{-1}[(-\infty,\delta)] \cap \xxx_0,\, N \in \nn^+\},
		$
		is a dense subset of $C(\rrm,\rrn)$.  This gives (i).  Moreover, by construction, every $g \in \xxx_1$ admits a representation satisfying (iii) and (iv).
		Furthermore, since $W_{J}\circ \sigma \bullet \dots \circ \sigma \bullet W_1 \in \xxx_2$ and by construction there exists some $g \in \xxx_1$ for which
		$
		d_{ucc}\left(
		W_{J}\circ \sigma \bullet \dots \circ \sigma \bullet W_1
		,g
		\right)<\delta,
		$; then (ii) holds.  
	\end{proof}
	\begin{proof}[{Proof of Corollary~\ref{cor_Constrained_archs_are_Univ}}]
		Since each $F_n$, for $n=1,\dots,N$, is a continuous function from $C(\rrm,\rrn)$ to $[0,\infty]$ then each $F_n^{-1}\left[[0,C_n)\right]$ is an open subset of $C(\rrm,\rrn)$.  Since the finite intersection of open sets is itself open, then 
		$\cap_{n=1}^N F_n^{-1}\left[[0,C_n)\right]$ is an open subset of $C(\rrm,\rrn)$.  Since there exists some $f_0\in C(\rrm,\rrn)$ satisfying~\eqref{eq_constraint_compatability} then $U$ is non-empty.  	Since $\arch$ has the UAP on $C(\rrm,\rrn)$ then $\arch \cap U$ is dense in $U$.  
		
		Fix $b \in \rrm$ with $b_i>0$ for $i=1,\dots,m$ and set $A=I_m$.  
		Since $\sigma$ is a transitive activation function then Corollary~\ref{cor_approx_loc_to_global} applies and therefore the set 
		$
		\left\{\Phi^N_{I_m,b}(f): \, f \in \NN \cap U\right\}
		$ is dense in $C(\rrm,\rrn)$.  Therefore (i)-(iv) hold.  
	\end{proof}
	\begin{proof}[{Proof of Corollary~\ref{cor_Composition_Lp}}]
		Let $S(x)=\sigma\bullet(x+b)$ and let $B\triangleq \left\{x \in \rrm: \sigma(x)>x\right\}$.  By hypothesis $B$ is Borel and $\mu(B)>0$.  
		For each $i=1,\dots,m$ we compute $
		\sigma\bullet(
		x_i+b_i
		)>
		x_i
		+ b_i \geq x_i
		$.  Therefore, for $\mu$-a.e. every $x \in B$, $N \in \nn$ and each $i=1,\dots,m$
		$$
		S^N(x)_i \geq x_i + Nb_i.
		$$
		Since $b_i>0$ then $\lim\limits_{N \to \infty} S^N(x)=\infty$.  Therefore, the condition \citep[Corollary 1.3 (C2)]{Bayart} is met, and by the discussion following the result on \citep[page 127]{Bayart}, condition \citep[Corollary 1.3 (C1)]{Bayart} holds; i.e.:
		for every non-empty open subset $U,V\subseteq L^1_{\mu}(\rrm,\rrn)$ there exists some $N_{U,V}\in \nn$ such that
		\begin{equation}
		\Phi_{I_m,b}^{N_{U,V}}(U)\cap V \neq \emptyset
		\label{eq_topo_trans}
		.
		\end{equation}
		By Lemma~\ref{lem_composition_Lp_basic_welldefinedness_and_norm_computation}, the map $\Phi_{I_m,b}$ and therefore the map $\Phi_{I_m,b}^{N_{U,V}}$ is continuous.  Thus, $(\Phi_{I_m,b}^{N_{U,V}})^{-1}[V]$ is a non-empty open subset of $L^1_{\mu}(\rrm,\rrn)$ and therefore 
		$U \cap (\Phi_{I_m,b}^{N_{\xxx,U}})^{-1}[V]$ is a non-empty open subset of $U$.  Taking $U=\operatorname{Ball}_{L^1_{\mu}(\rrm,\rrn)}(g,\delta)$ and $V=\operatorname{Ball}_{L^1_{\mu}(\rrm,\rrn)}(f,\epsilon)$ we obtain the conclusion.
	\end{proof}
	\begin{proof}[{Proof of Corollary~\ref{cor_examples_of_good_activations_Lp}}]
		By Proposition~\ref{prop_examples_of_good_activations} and the observation in its proof that $\sigma(x)>x$ we only need to verify that $\sigma$ is Borel bi-measurable.  Indeed, since $\sigma$ is continuous and injective then by \citep[Proposition 2.1]{HoffmannContinuityOfInverseStrictlyMonotone}, $\sigma^{-1}$ exists and is continuous on the image of $\sigma$.  Since $\sigma$ was assumed to be surjective then $\sigma^{-1}$ exists on all of $\rr$ and is continuous thereon.  Hence, $\sigma^{-1}$ and $\sigma$ are measurable since any continuous function is measurable.  
	\end{proof}
	\begin{proof}[{Proof of Theorem~\ref{cor_Explanation_of_Depth}}]
		Fix $A=I_m$ and $b\in \rrm$ with $b_i>0$ for $i=1,\dots,m$.  	
		Since $int({\co{\fff}})$ is a non-empty open set then there exists some $f \in int({\co{\fff}})$ and some $\delta>0$ for which 
		$$
		\operatorname{Ball}_{L^1_{\mu}(\rrm)}(f,\delta)\triangleq \left\{g \in L^1_{\mu}(\rrm):\, \int_{x \in \rrm} \|f(x)-g(x)\|d\mu(x)<\delta\right\}
		$$ is an open subset of $int({\co{\fff}})$.  Since $\co{\fff}\cap \operatorname{int}(\co{\fff})$ is dense in $\operatorname{int}(\co{\fff})$ then its intersection with any non-empty open subset thereof is also dense; in particular, $\co{\fff}\cap \operatorname{Ball}_{L^1_{\mu}(\rrm)}(f,\delta)$ is dense in $\operatorname{Ball}_{L^1_{\mu}(\rrm)}(f,\delta)$.  Since $\sigma$ is $L^1$-transitive then (iii) follows from Corollary~\ref{cor_Composition_Lp}.  
		
		Since $L^1_{\mu}$ is a metric space then $\left\{\operatorname{Ball}_{L^1_{\mu}(\rrm)}(g,\delta):\, g \in L^1_{\mu}(\rrm),\, \delta>0\right\}$ is a base for the topology thereon.  Therefore, Corollary~\ref{cor_Composition_Lp} implies that for any two non-empty open subsets $U,V \in L^1_{\mu}(\rrm)$ there exists some $N_{U,V}\in \nn$ satisfying $\Phi^{N_{U,V}}_{I_m,b}(U)\cap V \neq \emptyset$.  Hence, $\Phi_{I_m,b}$ is topologically transitive on $L^1_{\mu}(\rrm)$, in the sense of \citep[Definition 1.38]{LinearChaossogooodos}.  Moreover, since $\Phi_{I_m,b}$ is a continuous linear map then Birkhoff's transitivity theorem, as formulated in \citep[Theorem 2.19]{LinearChaossogooodos}, applies and therefore $\Phi_{I_m,b}$ is a hypercylic operator on $L^1_{\mu}(\rrm)$.  Therefore, \citep[Proposition 5.8]{LinearChaossogooodos} implies that $\|\Phi_{I_m,b}\|_{op}>1$.  Setting $\kappa\triangleq \|\Phi_{I_m,b}\|_{op}$ yields (ii).  
		
		It remains to show the approximation bound of described by (i).  Fix $f \in L^1_{\mu}(\rrm)$.  Since $L^1_{\mu}(\rrm)$ is a Banach space then it has no isolated points and since $\Phi_{I_m,b}$ is a hypercylic operator then Birkhoff's transitivity theorem, as formulated in \citep[Theorem 2.19]{LinearChaossogooodos}, implies that there exists a dense $G_{\delta}$-subset $HC(\Phi_{I_m,b})\subseteq L^1_{\mu}(\rrm)$ such that for every $g \in HC(\Phi_{I_m,b})$ the set $\{\Phi^N_{I_m,b}(g)\}_{N \in \nn}$ is dense in $L^1_{\mu}(\rrm)$.  Therefore, every non-empty open subset of $L^1_{\mu}(\rrm)$ contains some element of $HC(\Phi_{I_m,b})$.  In particular, there is some $g \in HC(\Phi_{I_m,b})\cap \operatorname{int}(\co{\fff})$ since $\operatorname{int}(\co{\fff})$ is a non-empty open subset of $L^1_{\mu}(\rrm)$.  

		Since $\co{\fff}\cap \operatorname{int}(\co{\fff})$ is dense in $\operatorname{int}(\co{\fff})$ then, in particular, $g \in \overline{\operatorname{int}(\co{\fff})}$.  Therefore, the conditions of \citep[Theorem 2]{darken1993rate} and~\citep[Equation (23)]{darken1993rate} are met, hence, for each $n \in \nn^+$ the following approximation bound holds
		\begin{equation}
		\inf_{f_i \in \fff,\, \sum_{i=1}^n \alpha_i=1,\, \alpha_i \in [0,1]}
		\int_{x \in \rrm} \left\|
		\sum_{i=1}^n \alpha_i f_i(x)-g(x)
		\right\|d\mu(x) \leq \frac{\sqrt{2\mu(\rrd)}}{\sqrt{n}}
		\label{cor_Explanation_of_Depth_proof_setup_conditions_darken_convex_hull_estimate}
		,
		\end{equation}
		
		Since $\{\Phi^N_{I_m,b}(g)\}_{N \in \nn}$ is dense in $L^1_{\mu}(\rrm)$ then there exists some $N \in \nn$ for which
		$\Phi^N_{I_m,b}(g) \in \operatorname{Ball}_{L^1_{\mu}(\rrm)}\left(f,\frac1{\sqrt{n}}\right)$.  Thus, the following bound holds
		\begin{equation}
		\int_{x \in \rrm} \|f(x)-\Phi^N_{I_m,b}(g)(x)\|d\mu(x) \leq 
		\frac1{\sqrt{n}}
		\label{cor_Explanation_of_Depth_proof_setuphypercylicit_bound_1}
		,
		\end{equation}	
		
		Since $\Phi_{I_m,b}$ is a continuous linear map from the Banach space $L^1_{\mu}(\rrm)$ to itself then it is Lipschitz with constant $\|\Phi_{I_m,b}\|_{op}$, where $\|\cdot\|_{op}$ denotes the operator norm, and by \citep[Corollary 2.1.2]{SinghManhasCompositionOpsFunSpaces1993} we have
		\begin{equation}
		\|\Phi_{I_m,b}\|^N_{op}=
		\left\|
		\frac{d(\sigma \bullet(\cdot + b))_{\#}\mu}{d\mu_M}
		\right\|_{\infty}^N
		\label{cor_Explanation_of_Depth_proof_Lipschitz_bound}
		.
		\end{equation}
		Moreover, by Lemma~\ref{lem_composition_Lp_basic_welldefinedness_and_norm_computation}, we know that the right-hand side of~\eqref{cor_Explanation_of_Depth_proof_Lipschitz_bound} is finite.  %
		Therefore~\eqref{cor_Explanation_of_Depth_proof_setuphypercylicit_bound_1} implies that for every $f_1,\dots,f_n \in \fff$, $\alpha_1,\dots,\alpha_n\in [0,1]$ with $\sum_{i=1}^n \alpha_i=1$, the following holds
		\begin{equation}
		\begin{aligned}
		& \int_{x \in \rrm} \left\|
		\Phi_{I_m,b}^N\left(\sum_{i=1}^n \alpha_i f_i\right)(x)
		-
		f(x)
		\right\|d\mu(x)
		\\ \leq  &
		\int_{x \in \rrm} \left\|
		\Phi^N_{I_m,b}\left(\sum_{i=1}^n \alpha_i f_i\right)(x)
		- \Phi^N_{I_m,b}\left(g\right)(x)
		\right\|d\mu(x)
		\\ &
		+
		\int_{x \in \rrm} \left\|
		f(x) - \Phi^N_{I_m,b}\left(g\right)(x)
		\right\|d\mu(x)
		\\
		&\leq 
		\left\|
		\Phi^N_{I_m,b}\right\|_{op}
		\left(
		\int_{x \in \rrm} 
		\left\|
		\sum_{i=1}^n \alpha_i f_i(x)
		- 
		g(x)
		\right\|d\mu(x)
		\right)
		\\
		& + 
		\int_{x \in \rrm} 
		\left\|
		\Phi^N_{I_m,b}\left(g\right)(x)
		- f(x)
		\right\|d\mu(x)
		\\
		&	\leq 
		\left\|
		\frac{d(\sigma\bullet(\cdot + b))_{\#}\mu}{d\mu_M}
		\right\|_{\infty}^N
		\left(
		\int_{x \in \rrm} 
		\left\|
		\sum_{i=1}^n \alpha_i f_i(x)
		- 
		g(x)
		\right\|d\mu(x)
		\right)
		+	
		\frac1{\sqrt{n}}
		.
		\end{aligned}
		\label{cor_Explanation_of_Depth_proof_Lipschitz_bound_eq_Lipschitz_bound_3}
		\end{equation}
		Combining the estimates~\eqref{cor_Explanation_of_Depth_proof_setup_conditions_darken_convex_hull_estimate} to~\eqref{cor_Explanation_of_Depth_proof_Lipschitz_bound_eq_Lipschitz_bound_3} we obtain
		\begin{equation}
		\begin{aligned}
		\inf_{f_i \in \fff,\, \sum_{i=1}^n \alpha_i=1,\, \alpha_i \in [0,1]} &
		\int_{x \in \rrm} \left\|
		\Phi^N_{I_m,b}\left(\sum_{i=1}^n \alpha_i f_i\right)(x)
		-
		f(x)
		\right\|d\mu(x)
		\\
		\leq & 
		\left\|
		\frac{d(\sigma\bullet(\cdot +b))_{\#}\mu}{d\mu_M}
		\right\|_{\infty}^N
		\left(
		\int_{x \in \rrm} 
		\left\|
		\sum_{i=1}^n \alpha_i f_i(x)
		- 
		g(x)
		\right\|d\mu(x)
		\right)
		+	
		\frac1{\sqrt{n}}
		\\ 
		\leq 
		&
		\left\|
		\frac{d(\sigma \bullet (\cdot +b))_{\#}\mu}{d\mu_M}
		\right\|_{\infty}^N
		\frac{\sqrt{2\mu(\rrd)}}{\sqrt{n}}
		+	
		\frac1{\sqrt{n}}
		\\
		=& \frac1{\sqrt{n}}\left(
		1 + \sqrt{2\mu(\rrm)}
		\right)
		.
		\end{aligned}
		\label{cor_Explanation_of_Depth_proof_Lipschitz_bound_eq_Lipschitz_bound_4_inf_taken}
		\end{equation}
		Since $\Phi^N_{I_m,b}$ is linear, then the right-hand side of~\eqref{cor_Explanation_of_Depth_proof_Lipschitz_bound_eq_Lipschitz_bound_4_inf_taken} reduces and we obtain the following estimate
		\begin{equation}
		\begin{aligned}
		\inf_{f_i \in \fff,\, \sum_{i=1}^n \alpha_i=1,\, \alpha_i \in [0,1]} &
		\int_{x \in \rrm} \left\|
		\sum_{i=1}^n \alpha_i 
		\Phi^N_{I_m,b}\left(
		f_i\right)
		(x)
		-
		f(x)
		\right\|d\mu(x)
		\leq
		\frac1{\sqrt{n}}\left(
		1 + \sqrt{2\mu(\rrm)}
		\right)
		.
		\end{aligned}
		\label{cor_Explanation_of_Depth_proof_Lipschitz_bound_eq_Lipschitz_bound_5_inf_taken}
		\end{equation}
		Therefore, the estimate in (i) holds.  
	\end{proof}
	For the statement of the next lemma concerns the Banach space of \textit{functions vanishing at infinity}.  Denoted by $C_0(\rrm,\rrn)$, this is the set of continuous functions $f$ from $\rrm$ to $\rrn$ such that, given any $\epsilon>0$ there exists some compact subset $K_{\epsilon}\subseteq \rrm$ for which
	$
	\sup_{x \in K_{\epsilon}}\|f(x)\|<\epsilon
	.
	$  As discussed in \citep[VII]{Vanishingatinfinity}, $C_0(\rrm,\rrn)$ is made into a Banach space by equipping with the supremum norm $\|f\|_{\infty}\triangleq \sup_{x \in \rrm} \|f(x)\|$.  
	\begin{lem}[Uniform Approximation of Functions Vanishing at Infinity]\label{lem_Uniform_Approxiation_for_Vanishing}
		Suppose that $\arch$ is a universal approximator on $C(\rrm,\rrn)$, then for every $f\in C_0(\rrm,\rrn)$ and every $\epsilon>0$ there exists $g_{\epsilon}\in C_0(\rrm,\rrn)$ with representation
		\begin{equation}
		f_{\epsilon}(\cdot) = \left(g_{\epsilon} e^{-\frac{b}{b - \|\cdot\|^2}} + a\right)I_{\|\cdot\|< b}
		+  \left(ae^{-
			\left|
			g_{\epsilon}(\cdot)
			\right|(\|x\|-b)}\right)I_{\|\cdot\|\geq b}
		\label{lem_Uniform_Approximation_for_Vanishing_functions}
		,
		\end{equation}
		the absolute value $\left|\cdot \right|$ is applied component-wise, $g_{\epsilon}\in \NN$, and $a,b>0$, and satisfying the uniform approximation bound $$
		\left\|
		f - f_{\epsilon}
		\right\|_{\infty} <\epsilon
		.
		$$
	\end{lem}
	\begin{proof}[{Proof of Lemma~\ref{lem_Uniform_Approxiation_for_Vanishing}}]
		Let $\arch$ be a universal approximator on $C(\rrm,\rrn)$, let $f \in C_0(\rrm,\rrn)$, and $\epsilon>0$.  Since $f$ vanishes at infinity then there exists some non-empty compact $K_{\epsilon,f}\subseteq \rrm$ for which $\|f(x)\|\leq \epsilon 2^{-1}$ for every $x \not\in K_{\epsilon,f}$.  By the Heine-Borel theorem $K_{\epsilon,f}$ is bounded and therefore there exists some $b^{\star}>0$ such that $K_{\epsilon,f}\subseteq \operatorname{Ball}_{\rrm}(0,b^{\star})\triangleq \left\{
		x \in \rrm:\, \|x\|< b^{\star}
		\right\}$.  Therefore,
		\begin{equation}
		\sup_{x \in \rrm - \operatorname{Ball}_{\rrm}(0,b^{\star})}
		\left\|
		f(x)
		\right\| <\epsilon 2^{-1}
		\label{lem_Uniform_Approxiation_for_Vanishing_eq_rephrasing_to_ball}
		.
		\end{equation}
		
		Since the bump function $x\mapsto e^{-1\frac{1}{1-x^2}}I_{|x|<1}$ is continuous, affine functions are continuous, $f\in C(\rrm,\rrn)$, and the composition and multiplication of continuous functions is again continuous then the function $x\mapsto \left[f(x)-\epsilon 2^{-1}\right]e^{\frac{b^{\star}}{b^{\star}-\|x\|^2}}I_{\|x\|<b^{\star}}$ is itself continuous.  Observe also that the set $\overline{\operatorname{Ball(0,b^{\star})}}= \left\{x \in \rrm:\, \|x\|\leq b^{\star}\right\}$ is closed and bounded, thus it is compact by the Heine-Borel theorem.  Since $\arch$ is a universal approximator on $C(\rrm,\rrn)$ for the topology of uniform convergence on compacts then there exists some $g_{\epsilon}\in \NN$ satisfying
		\begin{equation}
		\sup_{
			x 
			\in 
			\overline{\operatorname{Ball(0,b^{\star})}}
		}
		\left\|
		g_{\epsilon}(x) 
		- 
		\left[f(x)-\epsilon 2^{-1}\right]e^{\frac{b^{\star}}{b^{\star}-\|x\|^2}}I_{\|x\|<b^{\star}}
		\right\| <\epsilon 2^{-1}
		\label{lem_Uniform_Approxiation_for_Vanishing_controlling_onthe_compact_1}
		.
		\end{equation}
		Since $0\leq e^{-\frac{b^{\star}}{b^{\star}-\|x\|^2}} \leq 1$ for every $x \in \rrm$, then from~\eqref{lem_Uniform_Approxiation_for_Vanishing_controlling_onthe_compact_1} we compute
		\begin{equation}
		\begin{aligned}
		& \sup_{
			x 
			\in 
			\operatorname{Ball(0,b^{\star})}
		} \left\|
		g_{\epsilon}(x)e^{-\frac{b^{\star}}{b^{\star} - \|x\|^2}} 
		I_{\|x\|<b^{\star}}
		+ \epsilon 2^{-1} I_{\|x\|<b^{\star}}
		-
		f(x)
		\right\|\\
		\leq  & 
		\sup_{
			x 
			\in 
			\overline{\operatorname{Ball(0,b^{\star})}}
		} \left\|
		g_{\epsilon}(x)e^{-\frac{b^{\star}}{b^{\star} - \|x\|^2}} + \epsilon 2^{-1}
		-
		f(x)
		\right\|\\
		\leq &
		\sup_{
			x 
			\in 
			\overline{\operatorname{Ball(0,b^{\star})}
			}
		} 
		\left\|
		g_{\epsilon}(x) e^{-\frac{b^{\star}}{b^{\star}-\|x\|^2}}
		+ 
		\left(f(x) - \epsilon 2^{-1}\right)e^{\frac{b^{\star}}{b^{\star}-\|x\|^2}}e^{-\frac{b^{\star}}{b^{\star}-\|x\|^2}}
		\right\|\\
		\leq &
		\sup_{
			x 
			\in 
			\overline{\operatorname{Ball(0,b^{\star})}
			}
		} 
		e^{-\frac{b^{\star}}{b^{\star}-\|x\|^2}}
		\left\|
		g_{\epsilon}(x) 
		+ 
		\left(f(x) - \epsilon 2^{-1}\right)e^{\frac{b^{\star}}{b^{\star}-\|x\|^2}}
		\right\|\\
		\leq & 
		\sup_{
			x 
			\in 
			\overline{\operatorname{Ball(0,b^{\star})}
			}
		} 
		\left\|
		g_{\epsilon}(x) 
		+ 
		\left(f(x) - \epsilon 2^{-1}\right)e^{\frac{b^{\star}}{b^{\star}-\|x\|^2}}
		\right\|
		\\
		\leq & \frac{\epsilon}{2}
		.
		\end{aligned}
		\label{lem_Uniform_Approxiation_for_Vanishing_controlling_onthe_compact_2}
		\end{equation}
		Observe that, for every $x \in \rrm-\overline{\operatorname{Ball(0,b^{\star})}}$ we have $\|x\|-b^{\star}\geq 0$, $-|g_{\epsilon}(x)|\leq 0$ and therefore
		\begin{equation}
		0 \leq \epsilon 2^{-1} e^{-|g_{\epsilon}(x)| (\|x\|-b^{\star})} \leq \epsilon
		\label{lem_Uniform_Approxiation_for_Vanishing_controlling_network_outside_the_compact}
		.
		\end{equation}
		Combining~\eqref{lem_Uniform_Approxiation_for_Vanishing_eq_rephrasing_to_ball},~\eqref{lem_Uniform_Approxiation_for_Vanishing_controlling_onthe_compact_2}, and~\eqref{lem_Uniform_Approxiation_for_Vanishing_controlling_network_outside_the_compact} we compute the following bound
		\begin{equation}
		\begin{aligned}
		&\sup_{x \in \rrm} \left\|
		\left(g_{\epsilon}(x) e^{-\frac{b^{\star}}{b^{\star} - \|x\|^2}} +\epsilon 2^{-1}\right)I_{\|x\|<b^{\star}}
		+\epsilon 2^{-1} e^{-|g_{\epsilon}(x)|(\|x\|-b)}I_{\|x\|\geq b^{\star}}
		-
		f(x)
		\right\|
		\\
		\leq & \max \left\{
		\sup_{
			x 
			\in 
			\operatorname{Ball(0,b^{\star})}
		} \left\|
		g_{\epsilon}(x)e^{-\frac{b^{\star}}{b^{\star} - \|x\|^2}} 
		I_{\|x\|<b^{\star}}
		+ \epsilon 2^{-1} e^{-|g_{\epsilon}(x)|(\|x\|-b)}I_{\|x\|<b^{\star}}
		-
		f(x)
		\right\|
		\right.
		,\\ 
		&
		\left.
		\sup_{
			x 
			\in 
			\rrm-\operatorname{Ball(0,b^{\star})}
		} \left\|
		g_{\epsilon}(x)e^{-\frac{b^{\star}}{b^{\star} - \|x\|^2}} 
		I_{\|x\|<b^{\star}}
		+ \epsilon 2^{-1} e^{-|g_{\epsilon}(x)|(\|x\|-b)}I_{\|x\|<b^{\star}}
		-
		f(x)
		\right\|
		\right\}
		\\
		\leq & \max \left\{
		\epsilon
		,
		\sup_{
			x 
			\in 
			\rrm-\operatorname{Ball(0,b^{\star})}
		} \left\|
		g_{\epsilon}(x)e^{-\frac{b^{\star}}{b^{\star} - \|x\|^2}} 
		I_{\|x\|<b^{\star}}
		+ \epsilon 2^{-1} e^{-|g_{\epsilon}(x)|(\|x\|-b)}I_{\|x\|<b^{\star}}
		-
		f(x)
		\right\|
		\right\}
		\\
		= & \max \left\{
		\epsilon
		,
		\sup_{
			x 
			\in 
			\rrm-\operatorname{Ball(0,b^{\star})}
		} \left\|
		\epsilon 2^{-1} e^{-|g_{\epsilon}(x)|(\|x\|-b)}I_{\|x\|<b^{\star}}
		-
		f(x)
		\right\|
		\right\}
		\\
		\leq & \max \left\{
		\epsilon
		,
		\sup_{
			x 
			\in 
			\rrm-\operatorname{Ball(0,b^{\star})}
		} \left\|
		\epsilon 2^{-1} e^{-|g_{\epsilon}(x)|(\|x\|-b)}
		\right\|
		+
		\sup_{
			x 
			\in 
			\rrm-\operatorname{Ball(0,b^{\star})}
		} \left\|
		f(x)
		\right\|
		\right\}\\
		= & \max\{\epsilon,  \epsilon 2^{-1} +  \epsilon 2^{-1}\} = \epsilon
		.
		\end{aligned}
		\label{lem_Uniform_Approxiation_for_Vanishing_final_bound}
		\end{equation}
		Thus, the result holds.  
	\end{proof}
	\begin{proof}[{Proof of Theorem~\ref{thrm_sharpened_UAT}}]
		For each $\omega \in \Omega$, define the map $\Phi_{\omega}:C_0(\rrm,\rrn)\rightarrow C_{\omega}(\rrm,\rrn)$ by $\Phi_{\omega}(f)\triangleq \left(\omega(\|\cdot\|)+1\right)f$.  For each $f,g \in C_0(\rrm,\rrn)$ we compute
		\begin{equation}
		\begin{aligned}
		\left\|
		\Phi_{\omega}(f)
		-
		\Phi_{\omega}(g)
		\right\|_{\omega,\infty} = &
		\sup_{x \in \rrm}
		\frac{
			\left\|
			\Phi_{\omega}(f)
			-
			\Phi_{\omega}(g)
			\right\|
		}{
			\omega(\|\cdot\|)+1
		}
		\\
		= &
		\sup_{x \in \rrm} \frac{
			\left\|
			\left(\omega(\|\cdot\|)+1\right) f(x)
			-
			\left(\omega(\|\cdot\|)+1\right) g(x)
			\right\|
		}{
			\omega(\|\cdot\|)+1
		}
		\\
		= &
		\sup_{x \in \rrm} \frac{
			\left(\omega(\|\cdot\|)+1\right) 
			\left\|
			f(x)
			-
			g(x)
			\right\|
		}{
			\omega(\|\cdot\|)+1
		}
		\\
		= & \|f-g\|_{\infty}
		.  
		\end{aligned}
		\label{eq_thrm_sharpened_UAT_isometry}
		\end{equation}
		Therefore, for each $\omega\in \Omega$, the map $\Phi_{\omega}$ is an isometry.  For each $\omega\in \Omega$, define the map $\Psi_{\omega}:C_{\omega}(\rrm,\rrn)\rightarrow C_0(\rrm,\rr)$ by $\Psi_{\omega}(\tilde{f})\triangleq \frac1{\omega(\|\cdot\|)+1} \tilde{f}$.  For each $\tilde{f}\in C_{\omega}(\rrm,\rrn)$ and compute
		\begin{equation}
		\begin{aligned}
		\Phi_{\omega}\circ \Psi_{\omega}(\tilde{f})
		= &
		\Phi_{\omega} \left(
		\frac1{\omega(\|\cdot\|)+1} \tilde{f}
		\right)
		= &
		\left(\omega(\|\cdot\|)+1\right)\frac1{\omega(\|\cdot\|)+1} \tilde{f} = & \tilde{f}.
		\end{aligned}
		\label{eq_thrm_sharpened_UAT_bijection}
		\end{equation}
		Hence, $\Psi_{\omega}$ is a right-inverse of $\Phi_{\omega}$.  Since every isometry is a homeomorphism onto its image and since $\Phi_{\omega}$ is surjective isometry then $\Phi_{\omega}$ defines a homeomorphism from $C_0(\rrm,\rrn)$ onto $C_{\omega}(\rrm,\rrn)$.  In particular, $\Phi_{\omega}\left(C_0(\rrm,\rrn)\right)=C_{\omega}(\rrm,\rrn)$.  Therefore, 
		$$
		C_{\Omega}(\rrm,\rrn) 
		= 
		\bigcup_{\omega \in \Omega} 
		C_{\omega}(\rrm,\rrn)
		=
		\bigcup_{\omega \in \Omega} 
		\Phi_{\omega}\left(C_0(\rrm,\rrn)\right)=C_{\omega}(\rrm,\rrn).
		$$
		Hence, condition~\eqref{eq_density_condition} holds.  
		
		Since it was assumed that $\sup_{x \in \rrm} \|f(x)\|e^{-\|x\|}<\infty$ holds, then Lemma~\ref{lem_Uniform_Approxiation_for_Vanishing} applies, whence, 
		$$\left\{
		\left(fe^{-\frac{b}{b - \|\cdot\|^2}} + a\right)I_{\|\cdot\|< b}
		+  \left(ae^{-
			\left|
			f(\cdot)
			\right|(\|x\|-b)}\right)I_{\|\cdot\|\geq b}: 0<b,a, f \in \NN
		\right\}
		$$ 
		is dense in $C_0(\rrm,\rrn)$.  Therefore, the conditions for Theorem~\ref{thrm_Construction_Theorem} are met.  Hence,
		\begin{equation}
		\bigcup_{\omega \in \Omega} \Phi_{\omega}\left(
		\left\{
		\left(fe^{-\frac{b}{b - \|\cdot\|^2}} + a\right)I_{\|\cdot\|< b}
		+  \left(ae^{-
			\left|
			f(\cdot)
			\right|(\|x\|-b)}\right)I_{\|\cdot\|\geq b}: 0<b,a, f \in \NN
		\right\}
		\right)
		\label{eq_thrm_sharpened_UAT_subarchitecture}
		\end{equation}
		is dense in $C_{\Omega}(\rrm,\rrn)$.  By definition,~\eqref{eq_thrm_sharpened_UAT_subarchitecture} is a subset of $\NN[\arch[\Omega]]$ and therefore $\NN[\arch[\Omega]]$ is dense in $C_{\Omega}(\rrm,\rrn)$.  Hence, $\arch[\Omega]$ is a universal approximator on $C_{\Omega}(\rrm,\rrn)$.  
	\end{proof}
	\begin{proof}[Proof of Proposition~\ref{prop_limitation_PWL_UAs}]
		For each $k,m\in \nn$ with $n\leq m$, we have that $\exp(-k t)>\exp(-mt)$ for every $t \in[0,\infty)$.  Thus, 
		\begin{equation}
		C_{\exp(-k \cdot)}(\rrm,\rrn)\subseteq C_{\exp(-m \cdot)}(\rrm,\rrn)
		\label{prop_limitation_PWL_UAs_proof_inductive_limit}
		,
		\end{equation}
		and the inclusion is strict if $n<m$.  Moreover, for $n\leq m$, the inclusion of each $i^k_m:C_{\exp(-n \cdot)}(\rrm,\rrn)$ into $C_{\exp(-m \cdot)}(\rrm,\rrn)$ is continuous.  Thus, $\left\{C_{\exp(-k \cdot)}(\rrm,\rrn),i^k_m\right\}_{n \in \nn}$ is a strict inductive system of Banach spaces.  Therefore, by \citep[Proposition 4.5.1]{JarchowLCSs} there exists a finest topology on $\bigcup_{k \in \nn} C_{\exp(-k \cdot)}(\rrm,\rrn)$ both making it into a locally-convex space and ensuring that each $C_{\exp(-k \cdot)}(\rrm,\rrn)$ is a subspace.  
		Denote $\bigcup_{k \in \nn} C_{\exp(-k \cdot)}(\rrm,\rrn)$ equipped with this topology by $C_{\Omega}^{LCS}(\rrm,\rrn)$.  
		
		If $f \in C_{\Omega}^{LCS}(\rrm,\rrn)$ then by construction there must exist some $K \in \nn$ such that $f \in C_{\exp(-K\cdot )}(\rrm,\rrn)$.  By \citep[Propositions 2 and 4]{DieudonneSchwartz}, a sequence $\{f_t\}_{t \in \nn}$ converges to some $f$ if and only if there exists some $K \in \nn$ and some $N_K \in \nn^+$ such that for every $t\geq N_K$ every $f_t \in C_{\exp(-K\cdot)}(\rrm,\rrn)$ and the sub-sequence $\{f_t\}_{t\geq N_K}$ converges in the Banach topology of $C_{\exp(-K\cdot)}(\rrm,\rrn)$ to $f$.  In particular, since $C_{\exp(-0\cdot)}(\rrm,\rrn)=C_0(\rrm,\rrn)$ then the function $f(x)\triangleq (\exp(-|x|),\dots,\exp(-|x|)) \in C_{\exp(-0 \cdot)}(\rrm,\rrn)$.  Since each $f \in \NN$ is either constant of $\sup_{x \in \rrm} \|f(x)\|=\infty$ then for any sequence $\{f_t\}_{t \in \nn} \in \NN$ there exists some $N_0 \in \nn^+$ for which the sub-sequence $\{f_t\}_{t \geq N_0}$ lies in $C_{\exp(-0\cdot)}(\rrm,\rrn)=C_0(\rrm,\rrn)$ if and only if for each $t \geq N_0$ the map $f_t$ is constant.  Therefore, for each $t \geq N_0$ we compute that
		$$
		\|f - f_t\|_{\exp(0\cdot ), \infty}
		=
		\|f - f_t\|_{\infty} 
		\geq \inf_{c \in \rrm} \sup_{x \in \rrm} |\exp(-|x|)- c| 
		> \frac1{2}.  
		$$
		Hence, $f_t$ cannot converge to $f$ in $C_{\Omega}(\rrm,\rrn)$ and therefore $\arch$ does not have the UAP on $C_{\Omega}(\rrm,\rrn)$.  
	\end{proof}
	\begin{proof}[{Proof of Corollary~\ref{cor_Existence_Linfinity}}]
		Let $X\triangleq \rr$ and $\xxx_0\triangleq \xxx\triangleq L^{\infty}(\rr)$.  Since every Banach space is a pointed metric space with reference-point its zero vector and since $\rr$ is separable then Theorem~\ref{thrm_Existence} applies.  We only need to verify the form of $\eta$ and of $\rho$.  Indeed, the identification of $B(\rr)$ with $L^1(\rr)$ and explicit description of $\eta$ is constructed in \citep[Example 3.11]{WeaverNice}.  The fact that $L^{\infty}(\rr)$ is barycentric follows from the fact that it is a Banach space and by \citep[Lemma 2.4]{godefroy2003lipschitz}.  
	\end{proof}
\end{appendices}

\end{document}